%% file: main.tex
\documentclass[11pt,a4paper]{article}
\usepackage[T1]{fontenc}
\usepackage[utf8]{inputenc}
\usepackage{lmodern,microtype}
\usepackage[margin=25mm]{geometry}
\usepackage{amsmath,amssymb,graphicx,booktabs,longtable,array,tabularx}
\usepackage{xcolor,enumitem,caption,placeins,fancyhdr,xurl}
\usepackage[hidelinks]{hyperref}
\hypersetup{pdftitle={Who Can Change Course? An Integrative Review of AI Sovereignty and Collective Learning Capacity},pdfauthor={Kim Phuc Tran}}
\title{\textbf{Who Can Change Course?}\\[3pt]\large An Integrative Review of AI Sovereignty\\and Collective Learning Capacity}
\author{Kim Phuc Tran\textsuperscript{1,2}}
\date{20 October 2026}
\begin{document}
\maketitle
\begin{center}\small
\textsuperscript{1}Univ. Lille, ENSAIT, ULR 2461 -- GEMTEX -- G\'enie et Mat\'eriaux Textiles,\\F-59000 Lille, France\\[4pt]
\textsuperscript{2}International Chair in Data Science \& Explainable Artificial Intelligence,\\International Research Institute for Artificial Intelligence and Digital Transformation,\\Dong A University, Da Nang, Vietnam\\[5pt]
Correspondence: \href{mailto:kim-phuc.tran@ensait.fr}{kim-phuc.tran@ensait.fr}
\end{center}
\begin{abstract}
AI sovereignty is commonly associated with control over data, infrastructure and technological dependencies. This integrative review examines how institutions convert those resources into capabilities that remain accountable and open to revision. It brings together 72 references on digital sovereignty, absorptive capacity, innovation systems and sociotechnical governance, using six reading dimensions: subject, object, scale, mechanism, legitimacy and evidence. The synthesis proposes five conditions of collective learning capacity: recognizing and retaining knowledge, transmitting it across institutions, protecting consequential challenge, authorizing correction and preserving feasible alternatives. Authoritative corrigibility connects corrective evidence to decision authority; retained option value concerns the practical alternatives that survive a technological commitment. Comparisons of the United States, China and France within Europe distinguish resources and formal arrangements from documented operation and corrective outcomes. Building on this review, Human-Centered Learning Mechanics (HCLM) organizes a descriptive analysis of OECD, IMF and World Bank indicators around information resources, absorption and institutional coordination. The indicators show different resource profiles and trajectories, while leaving the operation of correction and feasible exit largely unmeasured. HCLM connects these observations to questions for institutional investigation; its national-level parameters are not estimated. Purposeful source selection, uneven case coverage and aggregate measurement constrain the conclusions. The resulting research agenda combines reproducible indicator analysis with decision-level evidence of learning, correction and retained choice.
\end{abstract}
\noindent\textbf{Keywords:} AI sovereignty; collective learning capacity; absorptive capacity; institutional learning; HCLM; AI governance
\clearpage
\tableofcontents
\clearpage
\input{sections/01_review}

\input{sections/02_hclm}
\input{sections/03_indicators}
\input{sections/04_implications}

\input{sections/05_references}
\end{document}

%% file: sections/01_review.tex
\section{Introduction: what does sovereignty enable?}

National AI strategies increasingly cite domestic models, public compute infrastructure, and locally hosted cloud services as evidence of sovereignty. Assessing these claims requires knowing who sets a system's purpose, which parts of its technical stack are supplied or governed elsewhere, and whether a ministry can evaluate performance independently of the supplier. It also requires examining the means available to workers, citizens and local authorities to contest harmful uses. When technology, evidence or public purposes change, who can revise the decision to deploy? An inventory of national assets leaves these questions unanswered.

The emphasis on assets reflects substantial material constraints. Advanced chips, data centers, foundation models, research talent and patient capital are expensive, geographically concentrated and exposed to geopolitical pressure. Cloud contracts and proprietary interfaces raise switching costs; export controls or supply disruptions can remove options with little warning. Territorial ownership therefore provides only part of the capacity needed for practical autonomy. A domestically hosted model may depend on externally supplied chips, toolchains or updates. Regulators may have strong legal powers and too few qualified staff to exercise them. Research excellence may fail to reach the firms or public services expected to use it. This gap between possession and use is the conversion problem examined here.

Sovereignty has several meanings in the literature. Digital-sovereignty studies use it to describe state authority, industrial autonomy, personal self-determination, and collective control over infrastructure or data (Couture and Toupin 2019; Pohle and Thiel 2020; Hummel et al. 2021). Technology-sovereignty research examines the room governments retain for action within interdependent systems (Edler et al. 2023). Roberts (2024) asks whether control is legitimate. Mügge (2024) considers who benefits when reliance on non-European platforms falls: citizens, European firms, or public institutions. A single sovereignty score would obscure the differences among these claims.

Existing reviews by Hummel et al. (2021) and Fratini et al. (2024) map these uses of sovereignty. This critical integrative review asks how strategic resources become a publicly accountable capacity that can operate amid external dependence and institutional disagreement. Sovereignty scholarship addresses control and authority; work on absorptive capacity and national innovation systems explains how knowledge becomes useful; sociotechnical governance examines the institutions and affected groups involved in its use. Reading these fields together provides an account of conversion that each develops only in part.

The review examines competing definitions of AI sovereignty, the institutions that turn technological resources into usable capability, and the evidence of durable public capacity in the United States, China and France within Europe. Six reading dimensions organize the comparison: subject, object, scale, mechanism, legitimacy and evidence. Collective learning capacity connects these accounts by examining how institutions retain and share knowledge, protect consequential challenge, act on reasons for revision and preserve feasible alternatives. These decisions span institutions with disputed purposes and unequal power, extending the problem beyond a bounded organization pursuing a managerial objective.

The theoretical contribution links knowledge acquisition to the authority and practical means needed to revise an existing commitment. Authoritative corrigibility concerns whether relevant evidence can reach a body empowered to require a response and implement correction. Retained option value concerns whether a technically, contractually and organizationally feasible alternative remains available. Together, these concepts specify how knowledge absorption can support accountable changes in public technological commitments.

The United States, China and France within Europe illustrate different institutional arrangements for this conversion. Extensive US private investment and research capacity coexist with concentrated infrastructure markets. China combines industrial accumulation and national coordination with implementation at several levels. Publicly available English-language evidence covers Chinese finance, implementation and feedback unevenly. In France, EU regulation, markets and shared infrastructure shape national implementation. The comparison examines these arrangements within the periods and institutional settings covered by the sources.

The paper first develops the review and its theoretical synthesis. Section 2 explains the corpus and reading method; Sections 3 and 4 develop the literature synthesis and collective learning capacity; Section 5 compares the three institutional settings. Section 6 then introduces HCLM as an analytical extension of that foundation. Section 7 applies it to sourced indicators, followed by public-value considerations, implementation proposals and a research agenda.

\section{Review design and interpretive strategy}

\subsection{Review design and corpus}

The study combines a critical integrative review with a descriptive analysis of published indicators. The review relates concepts and explanatory arguments across digital sovereignty, absorptive capacity, innovation systems and sociotechnical governance (Whittemore and Knafl 2005; Torraco 2005; Snyder 2019). It develops an institutional account of AI sovereignty and applies the same analytical questions to the United States, China and France within Europe. HCLM then provides a framework for organizing the indicator analysis. The literature synthesis and numerical calculations have distinct evidential roles.

The selected corpus contains 72 references: 39 journal articles, including conceptual essays and commentaries, 11 books, nine institutional reports, five legal or policy instruments, three methodological sources, three datasets and two conference papers. These inventory categories are mutually exclusive, with methodological articles classified by their role. The accompanying inventory records one entry per reference, its source type, analytical role and use in the manuscript. Its primary-role labels comprise 36 conceptual or theoretical sources, 24 sources used for case evidence, eight used for indicator interpretation or data, three review-method sources and one HCLM background work. These labels describe how references contribute to the argument; they do not count independent empirical studies.

The retained works contribute definitions of sovereignty, explanations of dependency or institutional capability, comparisons of strategies and institutions, or evidence about implementation and public accountability. A source may inform one question without addressing the others. Methodological works support the review design. Legal records and datasets support the particular institutional or descriptive claims for which they are cited. The forthcoming HCLM book provides the background for the institutional adaptation developed in Section 6 and used to organize the indicator analysis in Section 7. Its role remains distinct from the sources documenting national institutions and observations.

\subsection{Source interpretation and synthesis}

The source is the inventory unit, while the substantive claim is the unit of comparison. Six dimensions organize the source map. Subject identifies who claims, exercises or contests control. Object identifies the resource, service or decision concerned. Scale locates the claim at organizational, national, regional or cross-border level. Mechanism records the proposed or observed process through which control or capability operates. Legitimacy records the justification of authority and the rights or interests involved. Evidence identifies the kind and scope of support available for the claim. Multiple actors and scales can be retained, while fields outside a source's remit limit its interpretation. A source-level row summarizes its analytical use; it is not a transcript of every passage or claim.

Appraisal is qualitative and specific to the claim. Conceptual arguments are considered in relation to their assumptions and explanatory scope; empirical findings in relation to their settings and designs; official records in relation to their date and institutional remit; and datasets in relation to their definitions, coverage and comparability. The inventory distinguishes a source's evidential basis from the interpretation developed in the manuscript. For example, an administrative announcement establishes a mandate, whereas an implementation evaluation supplies evidence about operation within its remit. Tracing a challenge through an authorized response and implemented correction requires records of the relevant decisions.

The synthesis is organized around comparisons within and across the literature clusters. Absorptive-capacity and innovation-system accounts inform recognition, retention and transmission of knowledge. Accountability and sociotechnical research inform protected challenge and authorized correction. Work on dependence and switching constraints informs feasible alternatives. These connections motivate the five conditions of collective learning capacity. Authoritative corrigibility and retained option value specify how reasons for revision can reach decision authority and remain practically actionable. The six reading dimensions organize interpretation of sources; the five conditions are theoretical propositions developed through the synthesis.

Potential disagreements are interpreted in relation to actor, object, scale, period and evidence. Differences between descriptive accounts of control and normative accounts of legitimate authority are distinguished from incompatible empirical claims. For an empirical discrepancy, the comparison considers the scope, method and directness of the evidence. Where competing interpretations remain unresolved, the conclusion retains the uncertainty. The strength and relevance of the evidence guide the judgment, rather than the frequency with which a position is cited.

\subsection{Comparative scope and evidential limits}

The United States, China and France within Europe illustrate contrasting institutional arrangements for converting resources into capability. US private scale and infrastructure concentration, China's coordination across institutions and sectors, and France's implementation within EU rules and shared resources bring different aspects of that process into view. Each case is examined through the same five conditions, distinguishing resources and formal arrangements from documented operation and traced corrective outcomes. The cases support comparison of mechanisms within the settings covered by the sources.

The corpus is purposeful and non-exhaustive. The available review record consists of the final reference inventory and source-level interpretations. A complete query history, candidate pool and passage-level coding trail were not retained, so initial selection cannot be fully reconstructed or independently assessed for screening consistency. Source appraisal follows the qualitative rules above. English-language emphasis and uneven documentation limit the comparisons, and missing evidence leaves a capacity unresolved rather than establishing its absence.

\subsection{Relationship to the HCLM indicator analysis}

Section 6 sets out the HCLM interpretation, and Section 7 specifies the published indicators, observation periods, coverage, transformations and treatment of missing values used in the numerical analysis. The numerical results describe the selected observations under those stated choices. Policy, infrastructure and co-operation indicators represent enabling conditions; they do not directly observe institutional entropy, the effectiveness of a corrective response or the practical availability of an exit route. The case evidence and the review's five conditions identify the additional institutional observations needed to examine those mechanisms. Reproducibility of the calculations concerns the data and transformations supplied, while the review inventory provides traceability of the literature synthesis.

\section{What sovereignty means in the AI debate}

\subsection{Competing claims to sovereignty}

Authors use AI sovereignty to ask who may control digital systems, what that control covers, and whose purposes it serves. Couture and Toupin (2019) trace the term across states, Indigenous communities, social movements, and technology users. Pohle and Thiel (2020) describe tensions among governmental authority, economic independence, and individual autonomy. Reviewing 341 publications, Hummel et al. (2021) identify meaningful control as central to data-sovereignty debates, while documenting disagreement over the agent, object, and values involved. AI sovereignty thus encompasses several governance projects whose use of the same term can conceal differences in purpose.

Fratini et al. (2024) provide the closest large-scale map of these projects. Their review of 271 peer-reviewed articles distinguishes rights-based, market-oriented, centralization, and state-based models. They find that none combines broad governance coverage with sufficient responsiveness to technological, geopolitical, and social change. Collective learning capacity addresses this limitation across the models they identify. It asks whether institutions can detect error, admit previously excluded knowledge, and revise decisions through accountable procedures.

Jurisdictional accounts emphasize states' legal authority over data flows, platforms and infrastructures linked to their territory. Such authority can protect rights and reduce exposure to cross-border dependency, although its practical reach depends on technical and organizational capability. Mueller (2020) argues that territorial assumptions fit networks poorly and may encourage fragmentation. Pohle, Nanni and Santaniello (2024) show how broad appeals to sovereignty can support incompatible interventions. A jurisdictional claim must therefore specify how legal authority reaches the material stack and whose autonomy it protects.

In data-sovereignty arguments, control can concern ownership, access rules, withdrawal, community governance, or influence over data interpretation. These forms of control have different practical limits. A publicly owned dataset may remain tied to a proprietary platform, and formal consent may confer little influence over downstream model development. National repositories can also exclude minority knowledge. Calzati and van Loenen (2025) propose data commoning to examine ongoing data governance. Their account locates meaning and value in the social relations through which data are produced and used.

\subsection{Strategic autonomy across an interdependent stack}

Technology-sovereignty research examines the constraints imposed by external control of supply chains, standards and intellectual property. Edler et al. (2023) describe an innovation-policy objective of preserving room for action within interdependence. Advanced semiconductors, energy, cloud, models, data and applications span national boundaries, as does the knowledge on which innovation depends. The distribution of those dependencies matters: Farrell and Newman (2019) show how actors occupying central positions in global networks can use asymmetric relationships as leverage.

Layered accounts identify where such dependence occurs. Sheikh (2022) shows how strength in regulation or applications can coexist with weakness in semiconductors or hyperscale cloud. Belli (2026) emphasizes agency under constraint, provider switching, interoperability, and strategic openness. Plantin et al. (2018) explain how platforms have acquired infrastructural characteristics as public infrastructures have become platformized. Owning a facility may still leave its operation dependent on a supplier's software, updates, security practices, and pricing.

Territorial jurisdiction derives from public law. Platform power can constrain public action through technical architecture, contracts, and market position, even where territorial jurisdiction remains clear. Leaving a cloud provider illustrates the problem. An essential service may need applications rewritten for a different API, large datasets transferred, identity and monitoring services replaced, fine-tuning assets rebuilt, and staff trained on a new stack. Proprietary interfaces and weak portability increase these costs (Opara-Martins et al. 2016). A state can retain legal jurisdiction over a service while lacking a practicable way to change providers without disrupting delivery. Its options then depend on infrastructure, contracts, and public law.

European policy scholarship examines this tension between dependence and public authority. Floridi (2020) argues that dependence can limit the EU's ability to set and defend its values. Roberts et al. (2021a) find that policy statements combine appeals to European values with economic and geopolitical goals. Those appeals leave unresolved who chooses among competing values and who bears industrial-policy costs. Mügge (2024) finds that EU AI-sovereignty discourse has generally placed greater weight on jurisdictional independence and competitiveness than on citizen sovereignty. Reduced reliance on non-European platforms may strengthen European firms and public institutions without extending the agency of workers, public bodies, or communities.

Roberts (2024) argues that digital sovereignty requires a normative judgment about the legitimacy of control. Large technology firms supply many critical AI resources, although their internal governance is not ordinarily democratically authorized. Governments possess formal authority, but a particular use still requires justification, and meaningful public participation cannot be assumed. Capacity and legitimacy must therefore be examined together while remaining distinct evaluative questions. An index that treats technical capability and public accountability as substitutes would erase that distinction.

\subsection{National strategies, public-private ordering and imagined futures}

National AI strategies direct institutions towards particular priorities, assign roles to firms and public authorities, and make some risks more visible than others. Bareis and Katzenbach (2022) describe their influence as performative politics: strategic narratives help bring particular versions of AI into being. Hälterlein's (2025) study of 'AI made in Germany' shows how human-centered language operates within an ordoliberal political rationality, with its own assumptions about governance and markets. Ethical vocabulary can influence practice. Its presence in a strategy, however, provides no evidence on its own that affected people can influence decisions.

Comparative research finds substantial variation within relations between markets and states. Radu (2021) identifies hybrid governance in national strategies, where public and private actors jointly define priorities and oversight. Cath et al. (2018) examine differing assumptions about the 'good society' in US, EU, and UK approaches. Hine and Floridi (2024) relate American and Chinese policy traditions to their respective values and political institutions, while Roberts et al. (2021b) describe China's approach as combining economic development, international competition, social governance and ethical debate. These studies make strategies informative about institutional ambitions; implementation, workplace effects and opportunities for public challenge require further evidence.

Table 1 summarizes the contributions and limits of these traditions. Each explains aspects of asset control, legal authority, industrial autonomy or public legitimacy, while leaving questions about their interaction. The synthesis developed next examines how a political community uses technical resources and external knowledge, responds to disagreement and failure, and maintains public accountability for the power it acquires.

{\small
\setlength{\tabcolsep}{3.5pt}
\renewcommand{\arraystretch}{1.12}
\begin{longtable}{@{}>{\raggedright\arraybackslash}p{0.17\textwidth}>{\raggedright\arraybackslash}p{0.24\textwidth}>{\raggedright\arraybackslash}p{0.27\textwidth}>{\raggedright\arraybackslash}p{0.27\textwidth}@{}}
\caption{Traditions brought together in the integrative review}\label{tab:traditions}\\
\toprule
\textbf{Tradition} & \textbf{Representative sources} & \textbf{Analytical contribution} & \textbf{Unresolved issue} \\
\midrule
\endfirsthead
\multicolumn{4}{l}{\small\tablename~\thetable{} (continued)}\\
\toprule
\textbf{Tradition} & \textbf{Representative sources} & \textbf{Analytical contribution} & \textbf{Unresolved issue} \\
\midrule
\endhead
\midrule
\endfoot
\bottomrule
\endlastfoot
Contested meanings & Pohle and Thiel (2020); Hummel et al. (2021); Fratini et al. (2024) & Distinguishes state authority, corporate control, individual rights and collective claims. & Typologies clarify positions but say less about how institutions revise them when technologies and dependencies change. \\
\addlinespace[5pt]
Layered and technology sovereignty & Sheikh (2022); Edler et al. (2023) & Locates dependence across resources, chips, networks, cloud, models and applications; rejects autarky. & An asset map can still overlook skills, organizational uptake and the practical ability to change course. \\
\addlinespace[5pt]
Legitimate authority & Roberts (2024); Mügge (2024) & Separates effective control from justified authority and asks who gains from sovereign projects. & Normative evaluation needs an account of how legitimate capability is built, maintained and corrected. \\
\addlinespace[5pt]
National strategy and imaginaries & Bareis and Katzenbach (2022); Radu (2021); Hine and Floridi (2024) & Shows how strategies join values, industrial priorities, and public-private ordering while performing desired futures. & Policy texts reveal intentions and institutional logics more reliably than implementation or social outcomes. \\
\addlinespace[5pt]
Absorptive and innovation capacity & Cohen and Levinthal (1990); Castellacci and Natera (2013) & Explains why external knowledge has value only when organizations and systems can recognize, assimilate, and apply it. & Extending firm-level theory to countries requires attention to institutional differences; AI-specific outcomes require more precise evidence than broad innovation indicators. \\
\addlinespace[5pt]
Sociotechnical and public-value governance & Selbst et al. (2019); Novelli et al. (2024); Zuiderwijk et al. (2021) & Treats institutions, affected groups, accountability, and feedback as part of system performance. & The literature is often sectoral; links to national capacity and cross-border dependence remain underdeveloped. \\
\addlinespace[5pt]
Adaptive governance and resilience & Folke et al. (2005); Boin and van Eeten (2013) & Explains learning under uncertainty, polycentric adjustment and continuity through disruption. & Adaptation can preserve unequal arrangements; protected challenge, legitimate revision and feasible exit require separate examination. \\
\addlinespace[5pt]
\end{longtable}
}
\noindent{\footnotesize Traditions overlap; the table records their analytical use in this review. Sources may contribute to more than one tradition.}

\section{Collective learning capacity and institutional conversion}

\subsection{Definition and relation to existing research}

Cohen and Levinthal (1990) define firm-level absorptive capacity as the ability to recognize, assimilate, and apply external knowledge. Prior knowledge shapes what a firm can absorb, so this capacity develops over time. National innovation-systems research extends the analysis to relations among education, research, firms, finance, standards, and policy (Lundvall 1992; Edler and Fagerberg 2017; Castellacci and Natera 2013). Mazzucato (2013) examines how public institutions build productive capacity and address market failure. These traditions explain knowledge production, absorption, and diffusion. At the national scale, they also raise questions about whose knowledge a divided polity recognizes and who may reverse a public decision.

Applying firm-level absorptive capacity to countries requires attention to contested purposes, porous boundaries and institutional delays that can protect legitimate interests. Adaptive-governance research examines learning across levels, diverse knowledge and institutional adjustment (Folke et al. 2005). Resilience research addresses continuity and reorganization through disruption (Boin and van Eeten 2013), while sociotechnical studies show how abstraction can obscure context and feedback (Selbst et al. 2019). These concerns inform the account of sovereignty developed here. Section 6 develops a learning-systems interpretation after the comparative synthesis.

Collective learning capacity is the effective ability of interconnected institutions to recognize and retain relevant knowledge, transmit it across institutional boundaries, protect consequential challenge, authorize correction and preserve a feasible alternative or exit. These five conditions define the proposed account of accountable conversion. They operate among actors with different purposes and unequal authority. Assessment therefore concerns a specified decision, sector and institutional relationship within a political community.

The extension concerns the fate of knowledge that challenges a public technological commitment. Absorptive capacity explains how prior knowledge and organizational processes make external knowledge usable (Cohen and Levinthal 1990), while national innovation systems locate that process among organizations (Lundvall 1992; Castellacci and Natera 2013). For sovereignty, a further question arises: can that knowledge reach a responsible authority, and can the resulting decision be carried out? Research output and adoption figures describe resources for this process; its operation requires closer institutional evidence.

Institutional learning describes changes in shared knowledge, routines or rules through experience. Such changes may improve an existing program while leaving its purposes and commitments intact. Adaptive governance addresses adjustment across actors and levels (Folke et al. 2005). The proposed synthesis develops these concerns by specifying the decision rights and practical choices through which adjustment becomes accountable in AI procurement and deployment.

Institutional capacity concerns the acquisition, circulation and use of knowledge in decisions. Legitimacy concerns the authority for those decisions, the rights they protect and the distribution of benefits and burdens. Institutions may learn quickly under poorly justified authority or insecure rights; sound procedures may also coexist with inadequate means of implementation. Research on co-production and sociotechnical imaginaries likewise examines how knowledge and social order develop together (Jasanoff 2004; Jasanoff and Kim 2015).

Collective learning capacity varies by domain, institutional relationship, and territorial level. Competent revision of hospital procurement can coexist with dependence on frontier computing, while national expertise may fail to reach local authorities. Empirical assessment must specify the technological domain, the institutions through which knowledge moves, and the decision open to revision. Three process claims guide that assessment: knowledge must reach beyond the organization that first holds it for resources to become usable elsewhere; contestation must reach decision authority to change practice; and expertise and alternatives must survive the original commitment for correction to remain credible.

The distinction between capacity and its exercise is important. An institution can retain a credible correction route even when no incident requires its use. Conversely, an observed change of supplier or a cancelled deployment does not by itself demonstrate collective learning: the change may follow a price shock, budget cut or unrelated commercial decision. The empirical task is to trace how relevant knowledge affected an authorized response and whether the means of acting on that response remained available.

\subsection{Authoritative corrigibility and retained option value}

Authoritative corrigibility is the institutional capacity to connect a substantiated challenge to a decision that can change practice. The proposed mechanism requires a route for recording relevant evidence, access to a responsible authority, a reasoned response open to review, and practical powers to implement that response. The authority may reside in procurement, management, regulation or another competent body; it need not be a single national institution. An advisory panel can improve awareness without satisfying this condition if its findings never reach those who control contracts or deployment. Conversely, retaining a system after review can be consistent with corrigibility if the reasons and safeguards are documented and the decision remains revisable (Novelli et al. 2024; Roberts 2024).

Authoritative corrigibility concerns institutional decisions surrounding a sociotechnical system. The authority to respond remains subject to evaluation of its legitimacy, the rights affected and the distribution of burdens. Evidence for the mechanism would trace a particular concern to a response by an identified decision-maker and, where warranted, a change in deployment, contract or practice. An audit body supplies a possible route; its operation must be established from decision records.

Retained option value denotes the practical scope for changing course after resources have been committed. It extends the sovereignty concern with room for action (Edler et al. 2023) to choices that survive deployment. Options may include modifying a system, migrating to another provider, using an interoperable component, returning to a lower-technology process or ending an application while maintaining the underlying service. Their feasibility depends on usable data exports, retained expertise and affordable migration costs (Opara-Martins et al. 2016). Option value is used here as a qualitative institutional concept.

Retaining options has costs. Staff continuity, interoperable interfaces, migration testing and service fallbacks may reduce short-term efficiency, and not every dependency can be removed. An alternative counts as feasible only relative to a specified decision, acceptable service requirements, available resources and time horizon. A domestic provider is not necessarily easier to leave than an external one. Observable evidence would include a tested export or migration, a funded continuity plan and retained expertise. A list of alternative vendors or a contractual termination right is weaker evidence.

Two hypothetical hospital procurements illustrate the distinction. In the first, an oversight committee can order suspension, but the hospital lacks a safe service fallback. In the second, migration is technically straightforward, but no responsible body must consider a substantiated challenge. Correction is constrained by implementation in the first case and authorization in the second. Comparing such configurations would help establish how resources, decision authority and feasible alternatives interact.

Within the five conditions, recognition, retention and transmission concern the availability and circulation of knowledge. Authoritative corrigibility connects protected challenge to authorized correction through an effective institutional route. Retained option value describes how the fifth condition, feasible alternatives, survives over time. These connections can be examined at the level of a particular technological commitment.

\subsection{Who can change course? Decision authority and productive friction}

Institutions can miss relevant experience because it is never recorded or because records and speakers cannot reach decision-makers. Evidence may also reach people who lack the authority to act. A separate case arises when an authorized actor knows the adverse evidence yet chooses not to respond or change course; assessment then concerns the reasons for nonaction and whether they are documented and reviewable. Trade-secret restrictions impede inspection; employment dependence discourages reporting; limits on access or disclosure keep civic criticism out of decision records. Corporate AI-ethics staff, for example, may recognize risks without holding decision rights and may face personal costs for raising them (Ali et al. 2023). Feedback can inform action where speakers are protected, evidence is available, and an authorized body must respond. Otherwise, apparent success may reflect systematic false negatives: relevant evidence was excluded from the decision.

Friction is productive, in the proposed account, when scrutiny makes decision-relevant knowledge available, protects standing in an authorized review or preserves a practical option for later correction. These effects must be assessed alongside delay, resource costs and distribution of burdens. A procedure that confirms an existing decision can still be valuable if it provides justified assurance or protects a procedural right. Supplier and skill redundancy likewise imposes a present cost that may preserve continuity during a later failure. Productiveness cannot therefore be inferred solely from whether deployment accelerates or a decision changes.

Assessment should examine what scrutiny adds to the decision record and whether its demands are proportionate. Evidence may include a revised specification, a substantiated decision to continue, an additional safeguard, a documented refusal, transferable expertise or a tested fallback. Repeated review that adds no usable assurance, leaves responsibility unresolved or imposes avoidable burdens warrants examination as a possible conversion failure. A provider or sponsor should not be the sole judge of these effects; the grounds of the assessment need to remain open to review.

Collective learning occurs through interdependent relationships. External research, open standards, shared facilities and cross-border supply chains can expand capability when knowledge flows are reciprocal, and alternatives remain usable. Audit restrictions, impracticable switching and dependence on a vendor's essential expertise can constrain it. International rules for cross-border risks also support national capacity; governance at the two levels is complementary (Coeckelbergh 2025). Table 2 relates strategic inputs to institutional conversion, contestation and correction, and collective outcomes. Figure 1 shows their feedback relationships. Outcomes may reveal conversion failures, and correction may change subsequent inputs, so the four moments need not occur in a fixed sequence.

{\small
\setlength{\tabcolsep}{3.5pt}
\renewcommand{\arraystretch}{1.12}
\begin{longtable}{@{}>{\raggedright\arraybackslash}p{0.15\textwidth}>{\raggedright\arraybackslash}p{0.28\textwidth}>{\raggedright\arraybackslash}p{0.27\textwidth}>{\raggedright\arraybackslash}p{0.25\textwidth}@{}}
\caption{A mechanism-focused account of collective learning capacity}\label{tab:moments}\\
\toprule
\textbf{Moment} & \textbf{What it includes} & \textbf{Question for analysis} & \textbf{Typical failure} \\
\midrule
\endfirsthead
\multicolumn{4}{l}{\small\tablename~\thetable{} (continued)}\\
\toprule
\textbf{Moment} & \textbf{What it includes} & \textbf{Question for analysis} & \textbf{Typical failure} \\
\midrule
\endhead
\midrule
\endfoot
\bottomrule
\endlastfoot
Strategic inputs & Compute, data, research, capital, energy, skills and international knowledge flows & What resources are available, under which contractual and infrastructural dependencies? & Capacity remains latent when access is temporary, non-portable or concentrated. \\
\addlinespace[5pt]
Institutional conversion & Procurement competence, standards, sector intermediaries, public expertise and organizational learning & Can institutions evaluate, integrate and diffuse AI beyond frontier laboratories and large firms? & Assets may coexist with weak uptake, duplicated pilots and dependence on vendors for judgment. \\
\addlinespace[5pt]
Contestation and correction & Audit, worker and citizen knowledge, incident records, appeal, independent scrutiny and exit options & Can affected people expose error, trigger review and reach an authority able to change course? & Relevant feedback is excluded, responsibility for responding remains unclear, or an appeal receives no reasoned and reviewable response. \\
\addlinespace[5pt]
Collective outcomes & Durable capability, corrected practice, retained expertise and feasible alternatives & What capability remains, what changed after challenge, and does a feasible alternative survive? & Reported performance can improve while dependence deepens, expertise leaves, and correction remains infeasible. \\
\addlinespace[5pt]
\end{longtable}
}
\noindent{\footnotesize Power, interdependence and vendor concentration affect all four moments. Legitimacy, rights, distribution and ecological sufficiency are criteria for evaluating capacity's use; technical performance alone cannot establish them.}

\begin{figure}[htbp]
\centering
\includegraphics[width=\textwidth]{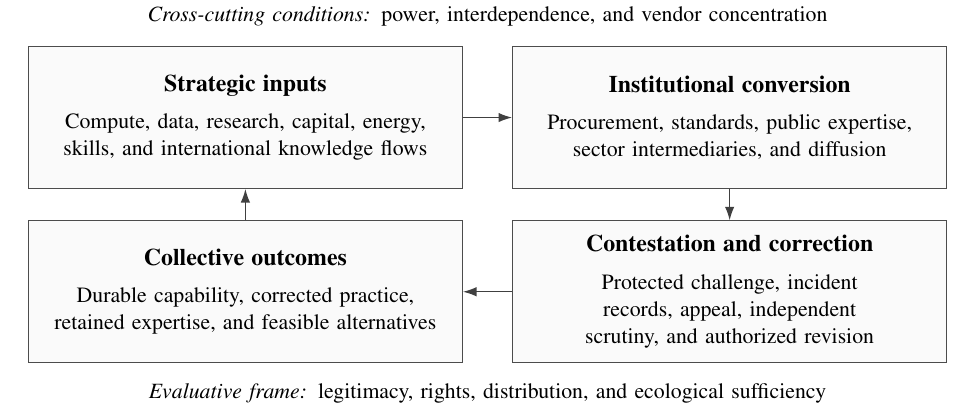}
\caption{Collective learning capacity and institutional conversion. Arrows indicate dependencies and feedback to earlier choices, without assuming a fixed sequence.}
\label{fig:learning}
\end{figure}
\FloatBarrier
\section{Three institutional configurations}
The literature-based comparison applies the five conditions in Section 4 to each setting. It distinguishes resources and formal arrangements, documented operation, and traced outcomes connecting evidence to a decision and its implementation. These levels call for different forms of support: a mandate identifies a responsible institution, whereas records of its decisions establish how it operates. The case accounts retain that distinction throughout.

\subsection{United States: private scale and public oversight across jurisdictions}

US financial and technical scale gives laboratories and firms substantial scope to develop and test AI systems. The 2026 AI Index records \$285.9 billion in private AI investment in 2025 and 59 notable models from US-based institutions (Stanford Institute for Human-Centered Artificial Intelligence 2026). Haag's (2025) comparison of international indices likewise documents substantial research, investment and computing resources. These inputs and outputs form the setting in which public institutions acquire expertise, exchange knowledge and evaluate deployments.

Several layers of this system are concentrated. A small number of large firms provide much of the cloud infrastructure, foundation-model access, and tools used by public agencies, researchers, and smaller firms. Customers may have limited contractual influence, and service access need not give them the expertise to inspect, alter, or leave a system. Plantin et al. (2018) explain the public significance of this concentration: platform services perform infrastructural functions across economic and public life. Roberts' (2024) distinction between control and legitimate authority then raises whether accountability at a critical layer reaches all affected groups.

US AI governance spans federal and state institutions. Federal research, procurement, and standards coexist with state regulation and labor protections that vary by jurisdiction. By the NCSL's July 2025 update, all 50 states had introduced AI legislation, and 38 had adopted or enacted around 100 measures across different fields (National Conference of State Legislatures 2025). States can test rules before Congress acts. Yet differences in definitions of high-risk systems, documentation duties, incident thresholds, and remedies complicate comparison and transfer. Common baselines for documentation, incident reporting, and portability would allow states to aggregate experience while retaining stronger local protections.

The retained US evidence is strongest on resources, platform dependence and formal governance activity. Access to inspection, protection for challenge and the effectiveness of remedies remain unevenly documented across jurisdictions. Public institutions' ability to evaluate and leave a service depends on retained expertise and contractual means. Shared incident records, review procedures and tested portability would allow closer examination of that capacity.

\subsection{China: coordinated deployment and limits of available evidence}

Studies of China's recent AI development describe coordinated investment and implementation across research, computing infrastructure, manufacturing, standards, and deployment (Roberts et al. 2021b; Hine and Floridi 2024). World Bank series record rapid increases over the past decade in R\&D intensity, researcher density, broadband, patent applications, and scientific publications (World Bank 2026). The 2026 AI Index reports leading publication and patent volumes and a large share of industrial-robot installation, but notes that reported private investment is not directly comparable with state-backed finance (Stanford Institute for Human-Centered Artificial Intelligence 2026). This alignment between investment and deployment may shorten some routes from laboratory research to industrial use.

The institutional arrangements described in these studies include national, provincial, municipal and organizational actors. Roberts et al. (2021b) and Hine and Floridi (2024) examine policy priorities and their institutional context. They identify intended coordination mechanisms whose implementation across universities, firms and public services requires examination at sector and organizational level.

China's data-localization rules and national data administration shape the conditions under which data become usable resources. Under the Personal Information Protection Law, state organs must store personal information domestically, and covered personal information held by critical information infrastructure operators or large processors must also be stored domestically; specified overseas transfers require a security assessment (Standing Committee of the National People's Congress 2021). The National Data Administration was created to coordinate data infrastructure and the integration, sharing, development and use of data resources (State Council of the People's Republic of China 2023). These arrangements establish territorial control and administrative goals for coordination and standardized data use. Their effects on cross-border knowledge flows, independent scrutiny and institutional learning require empirical investigation.

The China-focused sources document coordination mandates and deployment priorities more fully than individual challenges followed through a reasoned response and implemented correction. Differences in scope, language and units of observation constrain comparison with the other settings. Publication, patent and deployment counts describe activity, leaving the operation of corrective feedback unresolved.

Sector-level studies could connect China's manufacturing and deployment capacity to questions of dependence, distribution and public voice. Relevant records would trace access to infrastructure, learning during deployment, errors, appeals and the treatment of energy, labor, privacy and environmental costs. This evidence would complement indices designed around OECD institutions and address the uneven coverage of mobilization and corrective feedback.

\subsection{France within Europe: from EU rules to national and local implementation}

France's AI institutions cannot be examined apart from the European Union. National research bodies and public administration work within a regional system that also provides market scale, research networks, competition policy and common regulation. EU institutions can set lifecycle rules, support shared infrastructure and influence provider conduct across a large market. The AI Act establishes common obligations, but implementation depends on national authorities and nationally operated regulatory sandboxes (European Parliament and Council of the European Union 2024).

French ministries, regulators, regions and public services make implementation decisions about procurement, staffing, evaluation and the sharing of experience. Within the strategy periods it examines, the Cour des comptes (2025) evaluation documents research and computing achievements alongside uneven diffusion and organizational coordination difficulties. Tracing correction in a particular hospital or municipality would require decision records at that level.

European sovereignty scholarship explains the wider setting for those decisions. Sheikh (2022) identifies continuing dependence on cloud services and advanced semiconductors. Roberts et al. (2021a) link digital policy to protecting European values, while Mügge (2024) examines whether the beneficiaries are citizens, firms, or the EU as a geopolitical actor. Bradford (2023) explains how European regulation can influence conduct beyond its borders. These accounts illuminate regional power, but leave open how a hospital, municipality, or small manufacturer acquires the expertise to assess and integrate an AI system safely.

Recent adoption figures show wider use but not the depth of organizational change. AI use among OECD firms rose sharply between 2023 and 2025, and Eurostat reports that 20.0\% of enterprises in the EU survey population used an AI technology in 2025 (OECD 2026a; Eurostat 2025). The Eurostat population covers selected business sectors and enterprises with at least ten employees or self-employed persons. These measures cover both occasional tool use and extensive workflow change, with differences by firm size, sector and region obscured in aggregate figures. Uptake can precede evidence of effects on productivity, job quality, safety or public value. It may also increase dependence on external providers.

For France, the question is how rules, public investment and infrastructure alter organizational practice. Procurement expertise, sector-specific evaluation and organizational learning must endure when a provider changes or a contract ends. Smaller firms and public bodies may need data engineering, change management, legal guidance, and worker knowledge more immediately than a general-purpose model. The OECD's 2026 survey reports strong research and computing assets alongside gaps in investment, talent, and diffusion (OECD 2026b). Regulation can support learning by making failures comparable, reducing uncertainty and preserving routes for correction. Assessment should also include implementation costs, particularly where smaller organizations face overlapping duties.

Responsibilities fall at several levels of government. Semiconductor supply, frontier compute and competition policy require European coordination. Education, public-sector careers and much of industrial policy remain nationally organized. Hospital workflows, municipal services and vocational training rely heavily on local knowledge. Subnational policy brings broad commitments into contact with sector practice and local accountability (Liebig et al. 2024). Assigning responsibility therefore requires examining where scale economies, usable feedback and accountable authority can be combined.

{\small
\setlength{\tabcolsep}{3.5pt}
\renewcommand{\arraystretch}{1.12}
\begin{longtable}{@{}>{\raggedright\arraybackslash}p{0.15\textwidth}>{\raggedright\arraybackslash}p{0.267\textwidth}>{\raggedright\arraybackslash}p{0.267\textwidth}>{\raggedright\arraybackslash}p{0.267\textwidth}@{}}
\caption{Comparative synthesis of the three institutional configurations}\label{tab:cases}\\
\toprule
\textbf{Dimension} & \textbf{United States} & \textbf{China} & \textbf{France within Europe} \\
\midrule
\endfirsthead
\multicolumn{4}{l}{\small\tablename~\thetable{} (continued)}\\
\toprule
\textbf{Dimension} & \textbf{United States} & \textbf{China} & \textbf{France within Europe} \\
\midrule
\endhead
\midrule
\endfoot
\bottomrule
\endlastfoot
Primary source of scale & Private capital, hyperscale cloud, frontier laboratories and a large federal research base & State-guided investment, manufacturing depth, infrastructure rollout and large domestic deployment settings & EU market and research networks, national public investment, regulation and shared European facilities \\
\addlinespace[5pt]
Coordination pattern & Dense public-private networks with a large corporate role in infrastructure and research priorities, alongside federal-state variation & National policy direction with important provincial, municipal and firm-level implementation variation & Multi-level authority across EU institutions, the French state, regions, sector regulators and public services \\
\addlinespace[5pt]
Capacity emphasized in the reviewed evidence & Research and investment scale; documented public-private relationships and governance activity & Policy coordination mandates and deployment priorities across several institutional levels & EU legal arrangements, public research assets and national strategy implementation evidence \\
\addlinespace[5pt]
Constraint emphasized in the reviewed evidence & Platform dependence and jurisdictional variation; uneven evidence on effective correction and exit. & Uneven documentation of implementation and feedback at the level of individual decisions. & Implementation and diffusion gaps in the cited evaluations; effects of EU rules require examination within French organizations. \\
\addlinespace[5pt]
Learning-capacity question & Can public institutions retain and share technical learning while preserving independent evaluation and bargaining capacity? & What institutional routes allow disconfirming evidence to alter designated goals, and how can those routes be studied? & Can formal rules and shared assets be converted into practical capability for SMEs, workers and public bodies? \\
\addlinespace[5pt]
Territorial configuration & Federal system with state experimentation and nationally concentrated private infrastructure & National direction with provincial, municipal and firm-level implementation & French implementation nested in EU regulation, markets and shared infrastructure \\
\addlinespace[5pt]
\end{longtable}
}
\noindent{\footnotesize The entries summarize the arrangements and constraints documented in the retained sources. Table 4 examines the five learning conditions and identifies the evidence still needed to assess their operation.}

\subsection{Comparison across the five learning conditions}

Table 4 applies the same questions to each case and states the scope of the available evidence. Resource statistics describe inputs, legal and institutional records identify formal channels, and implementation studies document operation within particular settings. Assessment of the five conditions requires connecting these levels of evidence.

{\small
\setlength{\tabcolsep}{3.5pt}
\renewcommand{\arraystretch}{1.12}
\begin{longtable}{@{}>{\raggedright\arraybackslash}p{0.15\textwidth}>{\raggedright\arraybackslash}p{0.267\textwidth}>{\raggedright\arraybackslash}p{0.267\textwidth}>{\raggedright\arraybackslash}p{0.267\textwidth}@{}}
\caption{Evidence and limits for the five conditions of collective learning capacity}\label{tab:conditions}\\
\toprule
\textbf{Condition} & \textbf{United States} & \textbf{China} & \textbf{France within Europe} \\
\midrule
\endfirsthead
\multicolumn{4}{l}{\small\tablename~\thetable{} (continued)}\\
\toprule
\textbf{Condition} & \textbf{United States} & \textbf{China} & \textbf{France within Europe} \\
\midrule
\endhead
\midrule
\endfoot
\bottomrule
\endlastfoot
Recognition and retention & Research, model and investment resources are documented (Stanford HAI 2026; Haag 2025). Retention of public expertise after projects or supplier changes remains unassessed. & Research outputs and investment coordination are documented (World Bank 2026; Roberts et al. 2021b). Organizational memory and retention of deployment lessons require further evidence. & Research and compute assets are documented (Cour des comptes 2025; OECD 2026b). Retention of expertise requires evidence from individual services. \\
\addlinespace[5pt]
Transmission across institutions & Public-private relationships and state policy identify potential channels (Plantin et al. 2018; NCSL 2025). Systematic evidence of lesson reuse across agencies is limited. & Data coordination and sharing are administrative mandates (State Council 2023). Their operation and reuse across sectors require implementation evidence. & EU and national arrangements provide channels, while the strategy evaluation reports uneven diffusion (Cour des comptes 2025). Reach and uptake require organizational evidence. \\
\addlinespace[5pt]
Protected consequential challenge & Legislation describes varied oversight arrangements (NCSL 2025), while organizational research identifies barriers to raising concerns (Ali et al. 2023). Practical access requires comparison across jurisdictions. & Policy studies and data rules describe formal arrangements (Hine and Floridi 2024; Standing Committee 2021). Accessible and consequential challenge remains unassessed at the decision level. & The AI Act provides a legal framework for oversight (European Parliament and Council 2024). Effective access and the treatment of challenges in French deployments require implementation evidence. \\
\addlinespace[5pt]
Authorized correction & Formal governance activity is documented (NCSL 2025). Matched incidents tracing challenge, authorized response and implemented correction remain unassessed in the retained sources. & Coordination mandates are documented (State Council 2023). Comparable records connecting a specific challenge, competent authority and completed correction are limited in the retained corpus. & EU rules and national evaluation identify responsible institutions and implementation concerns (European Parliament and Council 2024; Cour des comptes 2025). Comparable records tracing challenges through to correction are needed. \\
\addlinespace[5pt]
Feasible alternatives or exit & Platform-dependence literature identifies switching constraints (Plantin et al. 2018). Comparable public-sector migration tests and continuity outcomes remain unassessed in the retained corpus. & Domestic deployment capacity and data administration are documented (Roberts et al. 2021b; State Council 2023). Usable supplier alternatives depend on migration costs, retained expertise and continuity, which require direct study. & Layered dependencies are documented (Sheikh 2022). Affordable and tested exit routes require evidence from individual French contracts. \\
\addlinespace[5pt]
\end{longtable}
}
\noindent{\footnotesize Each cell identifies available evidence and the remaining question. Unassessed capacity remains unresolved. Stanford HAI denotes Stanford Institute for Human-Centered Artificial Intelligence; NCSL denotes National Conference of State Legislatures; State Council and Standing Committee refer to the Chinese official sources listed in the references.}

Across the cases, accounts of technological resources and formal coordination are more extensive than records tracing a consequential challenge through to implemented correction and retained choice. The different institutional settings offer opportunities to study that sequence. In France, it crosses national and EU levels. Comparable sectoral studies would require consistent units, decision periods and records of both successful and unsuccessful correction.

%% file: sections/02_hclm.tex
\section{From the review to an HCLM interpretation}
\label{sec:hclm}

\subsection{Information, absorption and governed action}

The review identifies a conversion problem: access to knowledge and technology can expand while the ability to evaluate, revise or leave a particular deployment remains limited. Human-Centered Learning Mechanics (HCLM), developed in relation to control, calibration and governed action in industrial AI decision systems (Tran 2026), supplies a language for examining this problem dynamically. Its institutional application starts from the mechanisms identified in the review: knowledge enters organizations, is interpreted and retained, moves between them, and may lead to an authorized change in practice.

Three questions organize this application. What relevant information becomes available? Which organizational arrangements make it usable and preserve it over time? Which incentives and decision rights allow institutions to act on what they learn? Their connection can be expressed schematically as
\begin{equation}
\begin{aligned}
\text{AI sovereignty} ={}& \text{information flow}
  + \text{entropy control}\\
 &+ \text{incentive alignment}.
\end{aligned}
\label{eq:heuristic}
\end{equation}
The expression is a conceptual relationship: its equality and addition signs do not specify numerical equivalence or permit compensation between rights, knowledge and authority. Information flow concerns access to relevant knowledge and its transmission to actors who can use it. Entropy control refers here to limiting avoidable losses of institutional memory and coordination. Incentive alignment concerns funding, procurement and career arrangements that make knowledge sharing, challenge and justified correction practicable. Institutional entropy in this usage has no measured Shannon-entropy value.

The five conditions derived from the review give these broad terms institutional content. Recognition and retention concern whether incoming knowledge is understood and preserved. Transmission concerns its movement beyond the original holder. Protected challenge allows disconfirming evidence to enter the decision process. Authoritative corrigibility connects that evidence to a competent authority and an implementable response. Retained option value concerns the alternatives available when correction is warranted. These distinctions separate information resources from accountable capability.

Governance can support several of these mechanisms simultaneously. Common documentation can help institutions compare incidents; an appeal route can reveal overlooked effects; a portability requirement can preserve an alternative. Each also consumes time and resources. The relevant assessment concerns what a procedure contributes to evidence, rights, assurance and usable options alongside its costs. Delay alone cannot identify avoidable loss, and a stronger policy score cannot establish that a correction route works.

\subsection{A domain-specific capability balance}

A possible operational model must specify a decision domain and an observable outcome. Let $s$ identify a sector, service or procurement setting, and let $C_s(t)$ denote a defined stock of retained capability at time $t$. A candidate balance is
\begin{equation}
\frac{\mathrm{d}C_s(t)}{\mathrm{d}t}
  = \chi_s(t) I_s(t)-D_s(t),
\qquad 0\leq\chi_s(t)\leq 1.
\label{eq:balance}
\end{equation}
Here $I_s(t)$ is the rate of incoming capability opportunities, expressed in units compatible with $C_s$ per unit time; $\chi_s(t)$ is the fraction converted into retained capability; and $D_s(t)$ is the total rate of capability loss in the same units. Loss may arise from obsolescence or other attrition as well as avoidable failures of memory and coordination. An empirical analysis would separate the avoidable component, $D_s^{u}(t)$, from those other losses. Resources devoted to justified scrutiny, safety assurance or maintaining an alternative are evaluated by their effects on capability and rights, rather than classified automatically as losses. An operational study would need to define what counts as retained capability, when it is gained or lost, and how those events are recorded. Examples include the retention of independently usable evaluation procedures or the ability to maintain a service through a supplier change.

Equation~\eqref{eq:balance} is a proposed institutional adaptation. Its quantities would require sector-level measurement and estimation. It does not justify adding currency, researcher counts, computing hours and elapsed time, nor does it determine a national capability trajectory from a readiness index. Its analytical use is to separate potential input, conversion and loss when designing an empirical study. A recorded training activity may contribute to input; a team subsequently conducting an independent evaluation provides evidence about conversion; the loss of essential expertise after the project can reveal a failure of retention.

The same observed resource profile is consistent with different values of $\chi_s$ and $D_s$. Even an observed change in $C_s$ would identify only their net contribution without additional information. National resource indicators alone therefore cannot identify these terms. In the application below, the balance organizes the interpretation and the questions left open by the data. No values of $C_s$, $\chi_s$, $D_s$ or a national HCLM ratio are assigned to countries.

\subsection{Connecting concepts to available observations}

Table~\ref{tab:hclmcoverage} separates information about the structural setting from evidence of the mechanisms themselves. OECD R\&D and infrastructure components describe resources that can support information acquisition and use. Jobs and skills describe enabling human resources. Policy environment and international co-operation describe selected institutional conditions and connections. IMF preparedness brings a wider country set into view through related dimensions. World Bank series give a longer account of research effort, staffing, connectivity and scientific or technological activity.

None of these associations makes a component a direct measure of a learning mechanism. A researcher count does not show whether expertise survives a procurement cycle. Connectivity does not show whether a lesson travels between agencies. A legal framework does not show how a substantiated challenge was answered. The indicators are used at their stated level of measurement, and their interpretation returns to the review's distinction between resources, formal arrangements, operation and outcomes.

{\small
\setlength{\tabcolsep}{3.5pt}
\renewcommand{\arraystretch}{1.15}
\begin{longtable}{@{}>{\raggedright\arraybackslash}p{.17\textwidth}>{\raggedright\arraybackslash}p{.26\textwidth}>{\raggedright\arraybackslash}p{.25\textwidth}>{\raggedright\arraybackslash}p{.27\textwidth}@{}}
\caption{HCLM interpretation and the evidence needed for institutional assessment.}\label{tab:hclmcoverage}\\
\toprule
\textbf{Analytical concern} & \textbf{Available contextual indicators} & \textbf{Interpretation supported} & \textbf{Additional mechanism evidence}\\
\midrule\endfirsthead
\multicolumn{4}{l}{\small\tablename~\thetable{} (continued)}\\
\toprule
\textbf{Analytical concern} & \textbf{Available contextual indicators} & \textbf{Interpretation supported} & \textbf{Additional mechanism evidence}\\
\midrule\endhead
\midrule\endfoot\bottomrule\endlastfoot
Information resources & OECD R\&D and infrastructure; World Bank R\&D expenditure, researchers and broadband & Research effort, staffing and connectivity form part of the resource setting. & Access conditions, usable data, evaluation skills and the knowledge actually available to a specified service.\\\addlinespace[5pt]
Absorption and retention & OECD jobs and skills; IMF preparedness dimensions; researcher density & Human and organizational enabling conditions may support uptake. & Independent evaluations, workflow changes, retained expertise and capability surviving the end of external support.\\\addlinespace[5pt]
Transmission & OECD international co-operation; infrastructure and connectivity series & Selected international links and communication resources are documented. & Reuse of evaluations, lessons and data across institutions; evidence that these exchanges informed decisions.\\\addlinespace[5pt]
Avoidable capability loss & No direct measure in the retained national datasets & Aggregate profiles identify questions for investigation; they do not quantify loss. & Duplicated work, unusable records, loss of essential expertise and their effects on a defined capability.\\\addlinespace[5pt]
Governed correction & OECD policy environment and institutional records & Some formal arrangements and policy conditions are documented. & Protected challenge, an identified decision authority, a reviewable response and implemented action where warranted.\\\addlinespace[5pt]
Retained alternatives & No direct national measure in the retained datasets & Supplier availability alone leaves practical exit unresolved. & Tested migration, service continuity, portable records, retained skills and the cost of changing course.\\
\end{longtable}}

The application thus proceeds from published observations to bounded interpretation. A profile can motivate investigation of a particular conversion problem, but it cannot establish the cause of that problem. Comparing countries first requires comparable indicators and reference periods. Comparing collective learning then requires evidence about decisions and their institutional setting.

%% file: sections/03_indicators.tex
\section{Application to published indicators}
\label{sec:indicators}

\subsection{Sources, reference periods and calculations}

The application uses three sources with different purposes. The OECD.AI Index supplies five normalized component scores for 38 countries for its 2024 reference year: research and development, enabling infrastructure, jobs and skills, policy environment, and international co-operation (OECD 2026c, Figure 5.3). France and the United States appear in that table; China does not. The IMF AI Preparedness Index supplies a 2023 observation for each of the three cases (International Monetary Fund 2024). World Development Indicators (WDI) provide annual structural series for France, the United States and China over 2010--2023 (World Bank 2026).

The OECD values are transcribed at their published two-decimal precision. The IMF and World Bank observations were retrieved from their official APIs on 20 October 2026. The World Bank responses identify 8 October 2026 as the source update date. The accompanying files preserve the observations, retrieval addresses, indicator definitions and checksums. The 2010--2023 window maintains a common historical frame for the structural analysis; it is not a claim that 2023 is the latest year available for every underlying series. These observations also precede many recent generative-AI deployments and regulatory changes discussed in the case literature.

For the OECD profiles, each country's component is compared with the median of that same component across the 38 published observations. For country $c$ and component $k$, the reported difference is
\begin{equation}
d_{ck}=x_{ck}-\operatorname{median}_{j\in\mathcal{O}}x_{jk},
\qquad |\mathcal{O}|=38,
\label{eq:median}
\end{equation}
where $x_{ck}$ is the published normalized component and $\mathcal{O}$ is the country set in the OECD table. The differences retain the component's own scale. Comparing a country's position across components describes its profile, without treating a difference in policy score as an amount of research or a rate of knowledge absorption.

The World Bank analysis retains each series in its stated units. R\&D expenditure and high-technology exports are shares of GDP and manufactured exports, respectively. Researchers are measured per million people; fixed broadband subscriptions per 100 people. Resident patent applications and scientific and technical journal articles are converted to counts per million inhabitants using the population for the same country and year:
\begin{equation}
z_{ct}=10^6\frac{n_{ct}}{p_{ct}},
\label{eq:percapita}
\end{equation}
where $n_{ct}$ is the reported count and $p_{ct}$ the population. The article series uses fractional authorship counts, as defined by its source. For endpoint comparisons, 2013 is the common baseline and $t^*$ is the last nonmissing observation within 2020--2023. Relative change is
\begin{equation}
g_c=100\left(\frac{v_{c,t^*}}{v_{c,2013}}-1\right),
\label{eq:change}
\end{equation}
where $v$ denotes the relevant series in its original or per-capita units. The baseline is a descriptive comparison point, not an intervention date. Calculations use unrounded observations; rounding is applied for display. Missing observations remain missing. All reported endpoints are in 2023 except US researcher density, whose last observation within the window is in 2022. The calculations and plots use the observed values, with missing entries left unfilled.

\subsection{OECD profiles: resources and enabling conditions}

Table~\ref{tab:oecd-verified-components} and Figure~\ref{fig:oecd-profile} report the published profiles. France is above the sample median in R\&D, infrastructure, policy environment and international co-operation; its jobs-and-skills score equals the median. In the table's order, the differences are 0.025 for R\&D, 0.120 for infrastructure, 0.000 for jobs and skills, 0.270 for policy environment and 0.010 for international co-operation. The US profile is above the median in R\&D, infrastructure and policy environment, and equals it in jobs and skills and international co-operation at the published precision.

\input{sections/table_oecd}

\begin{figure}[htbp]
\centering
\includegraphics[width=.95\textwidth]{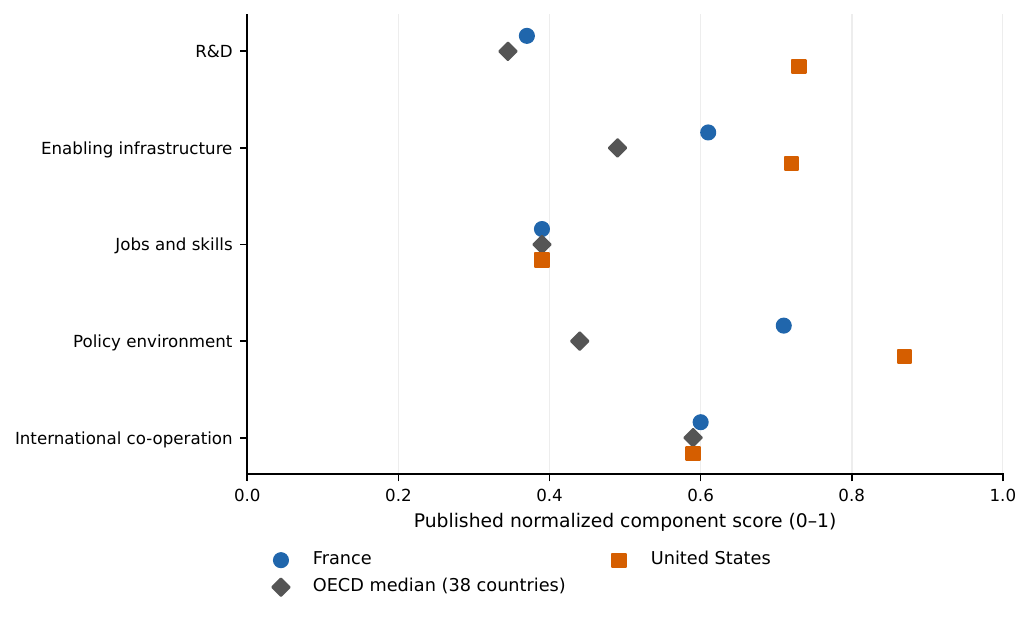}
\caption{Published OECD component profiles for France and the United States, with the 38-country component medians. Each row represents a distinct construct. China is outside the source table's coverage. Source: OECD (2026c), Figure 5.3.}
\label{fig:oecd-profile}
\end{figure}

In the HCLM interpretation, these profiles help distinguish resource conditions from the institutions through which resources may become useful. France's infrastructure and policy positions coexist with a jobs-and-skills score at the sample median. This observation motivates questions about the availability and circulation of practical expertise, particularly when read alongside the Cour des comptes (2025) findings on uneven diffusion. It does not demonstrate that policy has outpaced capability: the component scores have different content, and their distances cannot be treated as a measured conversion deficit.

The US profile documents a different combination of research resources and enabling conditions. A high R\&D component does not resolve whether public agencies retain independent expertise or can leave a particular service. The review identifies those questions through evidence about platform dependence, contracts and decision authority. Likewise, France's slightly higher international co-operation score than the US score cannot establish greater knowledge transfer between organizations. That mechanism requires records of what was exchanged and how it was used.

\FloatBarrier
\subsection{IMF preparedness across the three cases}

The IMF observations extend coverage to China (Table~\ref{tab:imf-verified-values}). The 2023 scores are approximately 0.698 for France, 0.771 for the United States and 0.635 for China. They describe preparedness across digital infrastructure, human capital and labour-market policies, innovation and economic integration, and regulation and ethics. Their broader coverage is useful for locating the three cases within a common measurement framework.

\input{sections/table_imf}

The ordering is specific to that framework and reference year. It does not establish a corresponding ordering of collective learning capacity, technological independence or legitimate authority. The index combines several enabling conditions that overlap conceptually with the OECD components, while its institutional assumptions and coverage differ. The HCLM reading therefore uses preparedness to frame questions about absorption and coordination, leaving the effectiveness of correction and retained alternatives to institutional evidence. The two indices are presented separately rather than combined into a new national score.

\FloatBarrier
\subsection{World Bank trajectories: resources and activity over time}

The World Bank series show that different structural conditions can move in different directions within a country (Table~\ref{tab:wdi-verified-endpoints}; Figure~\ref{fig:wdi-trends}). Between 2013 and 2023, France's R\&D expenditure moved from 2.23\% to 2.18\% of GDP, a decline of about 0.06 percentage points. Researcher density rose from 4,133.56 to 5,368.64 per million inhabitants, and fixed broadband from 38.73 to 48.67 subscriptions per 100 people. These observations describe greater research staffing and connectivity alongside broadly stable R\&D intensity.

Over the same period, US R\&D expenditure rose from 2.70\% to 3.45\% of GDP, about 0.75 percentage points. Researcher density rose from 3,815.21 in 2013 to 4,937.49 in 2022; fixed broadband rose from 30.00 to 38.09 subscriptions per 100 people by 2023. China's R\&D intensity rose from 1.96\% to 2.58\% of GDP, about 0.61 percentage points. Researcher density increased from 1,079.64 to 2,107.27 per million, and fixed broadband from 13.70 to 44.73 per 100 people. The magnitudes reflect different initial levels as well as subsequent change.

{\small\setlength{\tabcolsep}{4pt}
\input{sections/table_wdi}
}

Resident patent applications per million inhabitants declined between 2013 and 2023 from 222.58 to 198.94 in France and from 922.80 to 826.85 in the United States, while rising from 517.10 to 1,079.10 in China. Scientific and technical articles per million declined from 1,082.99 to 909.44 in France and from 1,349.34 to 1,279.40 in the United States, while rising from 257.86 to 661.16 in China. These are broad activity measures. Patent rules and incentives, publication fields, database coverage and the population denominator affect their interpretation; the counts do not measure research quality or the impact of AI.

The share of high-technology products in manufactured exports declined from 27.18\% to 22.38\% in France and from 31.59\% to 26.57\% in China, while rising from 20.19\% to 21.85\% in the United States. A change in this share can result from movements in either the numerator or the denominator, including shifts in industrial composition and global value chains. It cannot establish the depth of domestic AI adoption or the value retained by local firms.

\begin{figure}[p]
\centering
\includegraphics[width=\textwidth]{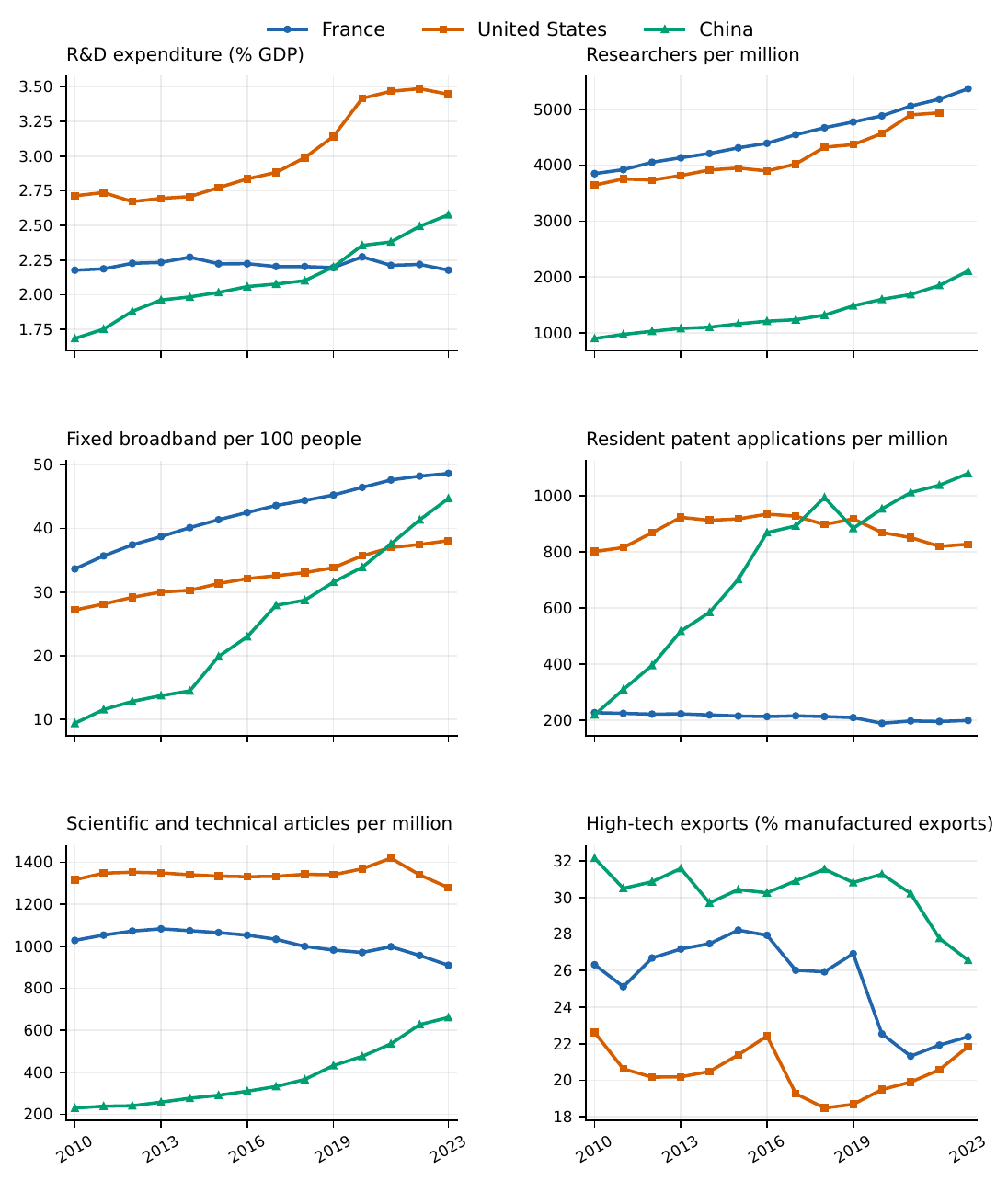}
\caption{Structural trajectories over 2010--2023 from the World Development Indicators vintage retrieved on 20 October 2026. Each panel retains its own measurement unit. Patent and article counts are population-normalized; missing observations are not interpolated. These broad series are not AI-specific measures.}
\label{fig:wdi-trends}
\end{figure}
\FloatBarrier

\subsection{What the application adds to the review}

The French profile illustrates why the review separates inputs from conversion. Increasing researcher density and broadband access coexist with a small decline in R\&D intensity and lower per-capita patent and article counts over the selected interval. Several explanations remain possible, including changes in field and sector composition, publication coverage, collaboration, incentives and the timing of research effects. The observations alone do not identify a fall in absorptive capacity or an increase in institutional dissipation. They locate a question that can be investigated by linking research and infrastructure provision to the organizations expected to use them.

For the United States, the rise in R\&D intensity and the OECD resource profile provide context for studying whether public institutions can retain evaluation capability within concentrated service markets. For China, growing research, connectivity and technological activity provide context for studying how lessons move between organizations and how particular decisions are revised. The same evidential standard applies in each case: greater resources are compatible with several institutional outcomes, and limited public documentation leaves a mechanism unresolved.

HCLM contributes a common set of questions across these observations. Are additional resources usable by the intended organizations? Is the resulting expertise retained and shared? Can evidence of failure reach an authority that can act, with a workable alternative available? The current datasets address the resource setting most directly. The proposed capability balance makes the missing conversion and loss terms explicit, while the review supplies the institutional content of the evidence required to study them. A further empirical study would need matched organizational records rather than a more elaborate aggregation of the same national indices.

%% file: sections/table_oecd.tex
\begin{table}[tbp]
\centering
\caption{Published 2024 OECD component scores and differences from the 38-country median.}\label{tab:oecd-verified-components}
\small
\begin{tabular}{lrrrrr}
\toprule
Component & France & US & Median & France $-$ median & US $-$ median \\
\midrule
R\&D & 0.37 & 0.73 & 0.345 & +0.025 & +0.385 \\
Infrastructure & 0.61 & 0.72 & 0.490 & +0.120 & +0.230 \\
Jobs and skills & 0.39 & 0.39 & 0.390 & +0.000 & +0.000 \\
Policy environment & 0.71 & 0.87 & 0.440 & +0.270 & +0.430 \\
International co-operation & 0.60 & 0.59 & 0.590 & +0.010 & +0.000 \\
\bottomrule
\end{tabular}
\par\smallskip\begin{minipage}{0.97\linewidth}\footnotesize Source: OECD (2026), \emph{The OECD.AI Index}, Figure 5.3, printed p.~32. Published component values have two-decimal precision. Medians use 38 countries. Differences compare each component with its own sample median; they are not absolute resource gaps. China is absent from this OECD table.\end{minipage}
\end{table}

%% file: sections/table_imf.tex
\begin{table}[htbp]
\centering
\caption{IMF AI Preparedness Index observations for 2023.}\label{tab:imf-verified-values}
\begin{tabular}{lrr}
\toprule
Country & Year & IMF AIPI \\
\midrule
France & 2023 & 0.697509 \\
United States & 2023 & 0.771269 \\
China & 2023 & 0.635454 \\
\bottomrule
\end{tabular}
\par\smallskip\begin{minipage}{0.90\linewidth}\footnotesize Source: IMF DataMapper, AI\_PI API, retrieved 20 October 2026. The index concerns AI adoption preparedness; it is not a direct measure of invention leadership or of collective learning capacity.\end{minipage}
\end{table}

%% file: sections/table_wdi.tex
\begin{longtable}{p{0.34\linewidth}rrrrr}
\caption{Structural indicators: 2013 and the last available observation within 2020--2023.}\label{tab:wdi-verified-endpoints}\\
\toprule
Indicator & Early year & Early value & Later year & Later value & Change (\%) \\
\midrule
\endfirsthead
\caption[]{Structural indicators (continued).}\\
\toprule
Indicator & Early year & Early value & Later year & Later value & Change (\%) \\
\midrule
\endhead
\midrule
\multicolumn{6}{r}{\footnotesize Continued on next page}\\
\endfoot
\bottomrule
\endlastfoot
\multicolumn{6}{l}{\textit{France}}\\*
R\&D expenditure (\% GDP) & 2013 & 2.23 & 2023 & 2.18 & $-2.48$ \\*
Researchers per million & 2013 & 4133.56 & 2023 & 5368.64 & 29.88 \\*
Fixed broadband per 100 & 2013 & 38.73 & 2023 & 48.67 & 25.65 \\*
Resident patent applications per million & 2013 & 222.58 & 2023 & 198.94 & $-10.62$ \\*
Scientific/technical articles per million & 2013 & 1082.99 & 2023 & 909.44 & $-16.03$ \\*
High-tech exports (\% manufactured exports) & 2013 & 27.18 & 2023 & 22.38 & $-17.65$ \\
\addlinespace
\multicolumn{6}{l}{\textit{United States}}\\*
R\&D expenditure (\% GDP) & 2013 & 2.70 & 2023 & 3.45 & 27.87 \\*
Researchers per million & 2013 & 3815.21 & 2022 & 4937.49 & 29.42 \\*
Fixed broadband per 100 & 2013 & 30.00 & 2023 & 38.09 & 26.98 \\*
Resident patent applications per million & 2013 & 922.80 & 2023 & 826.85 & $-10.40$ \\*
Scientific/technical articles per million & 2013 & 1349.34 & 2023 & 1279.40 & $-5.18$ \\*
High-tech exports (\% manufactured exports) & 2013 & 20.19 & 2023 & 21.85 & 8.21 \\
\addlinespace
\multicolumn{6}{l}{\textit{China}}\\*
R\&D expenditure (\% GDP) & 2013 & 1.96 & 2023 & 2.58 & 31.33 \\*
Researchers per million & 2013 & 1079.64 & 2023 & 2107.27 & 95.18 \\*
Fixed broadband per 100 & 2013 & 13.70 & 2023 & 44.73 & 226.51 \\*
Resident patent applications per million & 2013 & 517.10 & 2023 & 1079.10 & 108.68 \\*
Scientific/technical articles per million & 2013 & 257.86 & 2023 & 661.16 & 156.40 \\*
High-tech exports (\% manufactured exports) & 2013 & 31.59 & 2023 & 26.57 & $-15.88$ \\
\addlinespace
\end{longtable}
\noindent\footnotesize Source: World Development Indicators API, retrieved 20 October 2026; source last updated 8 October 2026. The panel covers 2010--2023. Percentage changes are $100(v_{\mathrm{later}}/v_{2013}-1)$ and use unrounded values. Missing years are not interpolated. Patents and articles are divided by the corresponding country-year population and expressed per million inhabitants. The US researcher endpoint is 2022; all other endpoints are 2023.\normalsize

%% file: sections/04_implications.tex
\section{Public value and legitimate use}

Technical capacity does not settle the question of legitimate use. An institution may recognize, transmit, and act on relevant knowledge while excluding those who bear the consequences. The assessment should examine how that capacity is distributed: who can contest decisions, whose purposes guide them, where risks fall and whether the arrangement remains open to revision.

Sen's (1999) capability approach asks what people can actually do with available resources. A training program may offer formal access while producing little capability if workers lack time, bargaining power, complementary tools or recognized credentials. Publicly funded compute raises a similar issue. Allocation rules, governance and benefit sharing can favor institutions that already have administrative and financial capacity (Ostrom 1990). Regional and linguistic inclusion therefore depends on the conditions under which resources can be used.

For people affected by an AI-supported decision, practical freedom includes knowing that the system was used, access to the relevant record and a protected route to an authority able to provide redress. Appeals can reveal recurring problems missed by developers, vendors or managers. Consultation after the design has been settled has less scope to influence the outcome. Accountability and participation support learning when their findings reach an authority able to remedy the problem (Novelli et al. 2024; Taylor et al. 2025).

What institutions accept as relevant knowledge also shapes their capacity to learn. The identity and interests of the people described as the human centre of AI deserve explicit attention (Ahn 2026). Sovereignty models also reflect the cultural and political assumptions of their settings (Keerthiraj and Misra 2026). In France, attention to regional languages, overseas territories, migration, disability, and professional knowledge is needed alongside assessment of French-language model performance. Data commoning can support shared stewardship where communities have authority over the purposes and conditions of reuse (Calzati and van Loenen 2025).

AI can alter task allocation, monitoring, discretion and bargaining power in the workplace (Cazzaniga et al. 2024). Workers also hold operational knowledge that can expose failure and improve safe practice. An approach limited to training them as end users fails to draw on that knowledge. Sectoral social dialogue, paid learning time, access to relevant information and joint evaluation should therefore form part of absorptive capacity. Job counts alone reveal little about whether work has become more skilled, autonomous or fairly rewarded (Acemoglu and Johnson 2023).

Assessing institutional capacity also requires accounting for the uneven environmental burdens of electricity, water, land and hardware used by compute (Crawford 2021). Consumption figures need context: the same energy or water use may support a widely reused public model, duplicate underused capacity, or sustain low-value inference in a water-stressed location. Timing, utilization, grid carbon intensity, and technical alternatives affect the assessment. Downstream applications and rebound extend the effects beyond data-center operations (Kaack et al. 2022), while limited reporting on direct, electricity-related, and supply-chain water use hampers local planning (Privette et al. 2026). Institutions need the ability to meter these burdens, compare smaller or shared alternatives, and stop workloads whose marginal public value no longer warrants the cost.

Disagreement may expose a design flaw, and maintaining a second supplier may preserve a credible exit route. The value of these arrangements depends on whether participants were heard, how evidence informed the decision and whether viable alternatives remained available. Evaluation should examine those outcomes alongside the time and resources required for review.

\section{Implications for policy, design and implementation}

Table~\ref{tab:policy} presents six policy design propositions and the evidence needed to evaluate them. The labels identify their basis: direct engagement with a mechanism in the literature (L), a literature-informed design inference (D), or a framework-derived institutional hypothesis (H). The final column relates each proposition to the four moments in Table 2.

{\small
\setlength{\tabcolsep}{3.5pt}
\renewcommand{\arraystretch}{1.12}
\begin{longtable}{@{}>{\raggedright\arraybackslash}p{0.15\textwidth}>{\raggedright\arraybackslash}p{0.33\textwidth}>{\raggedright\arraybackslash}p{0.34\textwidth}>{\raggedright\arraybackslash}p{0.13\textwidth}@{}}
\caption{Policy design propositions and evidence requirements}\label{tab:policy}\\
\toprule
\textbf{Priority} & \textbf{Institutional design} & \textbf{Evidence to retain and publish} & \textbf{Table 2 moment(s)} \\
\midrule
\endfirsthead
\multicolumn{4}{l}{\small\tablename~\thetable{} (continued)}\\
\toprule
\textbf{Priority} & \textbf{Institutional design} & \textbf{Evidence to retain and publish} & \textbf{Table 2 moment(s)} \\
\midrule
\endhead
\midrule
\endfoot
\bottomrule
\endlastfoot
Public-interest compute (H) & Allocate shared capacity by scientific, SME, and public-mission criteria; require workload portability and publish utilization, energy, and water use. & Access by firm size and region; switching time; concentration; lifecycle resource intensity; independently assessed public benefit. & Inputs \(\rightarrow\) conversion \(\rightarrow\) outcomes \\
\addlinespace[5pt]
Sectoral diffusion (H) & Fund regional teams that combine workflow analysis, data engineering, legal guidance, worker participation, and reusable evaluation. & Durable adoption and process outcomes after 12--24 months; worker capability; evidence reused across organizations. & Conversion \(\rightarrow\) outcomes \\
\addlinespace[5pt]
Talent circulation (H) & Use portable fellowships, joint appointments, public technical careers and paid continuing learning to move knowledge across institutions. & Retention and cross-sector movement; regional and demographic distribution; capability retained by host organizations. & Inputs \(\rightarrow\) conversion \\
\addlinespace[5pt]
Participatory procurement (D) & Stage high-impact purchases through baselines, affected-group panels, red-team testing, appeal design and credible exit conditions. & Incident closure; appeal resolution; access to logs; contract portability; supplier switching; avoided repeated learning costs. & Conversion \(\rightarrow\) correction \(\rightarrow\) outcomes \\
\addlinespace[5pt]
Governed data commons (L) & Give represented communities a role in purposes, access, quality, benefit sharing and withdrawal; use trusted access where open release would harm. & Coverage and documentation; governance participation; responsible reuse; distribution of benefits and risks. & Inputs \(\rightarrow\) correction \(\rightarrow\) outcomes \\
\addlinespace[5pt]
Public learning account (H) & Report failures and dependencies besides spending, compute hours, and adoption; link program renewal to evidence and independent review. & Correction time; concentration; redress; job quality; regional equity; public value; ecological cost and retained options. & Correction \(\rightarrow\) outcomes \(\rightarrow\) revised inputs \\
\addlinespace[5pt]
\end{longtable}
}
\noindent{\footnotesize L = direct engagement with a mechanism in the literature; D = literature-informed design inference; H = framework-derived institutional hypothesis. These labels identify the basis of the proposition, not an effect estimate or quality score. Each requires evaluation in the proposed setting, with measures disaggregated where lawful and meaningful.}

The public value of shared compute depends on its rules of access and operation. Allocation criteria can include scientific quality, public missions, smaller-firm experimentation, and language resources. Portable workloads, access to more than one operator, and published utilization would help institutions assess dependence. Reporting energy and water use by workload would also make it possible to examine whether the facility delivers reusable public benefit commensurate with its resource costs.

Hospitals, municipalities and small firms may lack the expertise to use shared compute in everyday work. Workflow analysis, data engineering, legal interpretation and worker knowledge are often needed before technical resources can be put to use. Regional intermediaries and portable technical careers could move this expertise between projects. Their contribution can be assessed through organizational changes and the capability retained after external support ends.

Procurement contracts can enable correction through staged testing, access to logs, protected input from workers and citizens, an appeal route, and agreed exit conditions. Sectoral or multilingual data commons offer a related means of giving represented communities authority over purpose, access, quality, benefit sharing and withdrawal. Trusted access may be preferable where open release would expose those communities to harm. Ethical and human-centered scholarship sets out many relevant duties; putting them into practice requires budgets, decision authority and enforceable consequences (Floridi et al. 2018; Jobin et al. 2019; OECD 2024; UNESCO 2022; Shneiderman 2022).

Comparative evaluation could examine whether common documentation allows agencies to reuse evidence, whether portability clauses permit supplier changes in practice, and whether participation influences renewal, redesign or termination. Alongside adoption and spending, relevant outcomes include redress, job quality, regional distribution and ecological cost.

A cross-government body could set evidence and interoperability standards while sector agencies and regions adapt them and test alternatives. Publishing failures and dependencies would let subsequent projects draw on earlier experience. Evaluation of this arrangement would ask whether new evidence actually led to corrections in policy and practice.

\section{Research agenda and limits}

\subsection{Evidence about institutional learning}

Empirical assessment should distinguish an available institutional capacity from its exercise in a particular episode. It requires evidence about relevant knowledge, access to decision authority, a reviewable response and the means of implementing a justified change. These processes provide the focus for the research agenda.

Process tracing could reconstruct the sequence from an incident or newly identified dependency to the institutional response. It would establish when the problem became visible, whether information traveled beyond the organization where it arose, and whether affected or knowledgeable people could challenge the decision without retaliation. The record should then show how the responsible authority answered and whether the system was implemented, revised, suspended, or ended. Researchers would also need to test whether a workable alternative supplier or method remained available. Gaps in this sequence would locate failures in detection, transmission, protected challenge, authorized response, or retained choice.

Longitudinal studies of hospitals, schools, municipalities, or small firms could follow procurement, deployment, and review to establish how expertise develops, how suppliers constrain choices, and which lessons survive a pilot. Comparative process tracing could examine whether procurement competence, worker participation, or independent audit is associated with durable use and public value. The sample should include canceled systems and decisions against automation: these may provide stronger evidence of learning than adoption alone.

A useful comparative design would follow similar procurement or deployment episodes over time, specifying in advance what counts as a consequential challenge, an authorized response and a feasible alternative. Records should include cases in which a system was retained after review, not only successful corrections or cancellations. Alternative explanations, including budget changes, supplier failure and legal compulsion unrelated to the reported challenge, need to be examined. Observing the proposed sequence would support a mechanism account; establishing the contribution of each condition would require additional comparison and attention to selection effects.

Because dependence arises through relationships, its measurement should examine them directly. Vendor-switching tests, contract analysis, contribution networks and supply-chain mapping could establish whether claimed alternatives are usable. Research should also follow capabilities across regional, national and European levels. National institutions organize budgets, education, taxation and much political accountability, while suppliers, standards and research networks cross those boundaries.

National averages can hide losses of occupational discretion, uneven infrastructure investment, gaps in language coverage and local concentrations of monitoring or environmental costs. Interviews, ethnography, deliberative studies and administrative records could examine these effects and establish whether participation changed a decision. Research outside the OECD is particularly needed because prevailing readiness frameworks may privilege one institutional model, overlooking capability organized through communities, informal institutions, or regions (Cachipuendo-Contero 2026).

Environmental studies should examine computing by workload and location as well as aggregate electricity demand. Grid, water, and hardware effects depend on when and where a system runs, its utilization, and the alternative it displaces. Linking procurement records with lifecycle assessment could relate technical performance and resource use to the public purpose served. Such studies should also measure rebound effects.

\subsection{Limits of the synthesis and indicator application}

The corpus is purposeful and its case coverage uneven. The source inventory supports inspection of the present synthesis, while the incomplete search record limits reconstruction of initial selection. The five learning conditions are theoretical propositions whose causal necessity, sufficiency and possible substitution remain untested. National indicators describe resources and enabling conditions across different reference periods; they cannot identify the effect of institutional correction. HCLM provides a way to formulate those mechanism questions, but the available data do not estimate its institutional parameters or support a causal national capability score.

Established interests may block relevant knowledge from influencing a decision even where technical and administrative capability is adequate (Ali et al. 2023; Acemoglu and Johnson 2023). This is a limitation of any account that infers learning from institutional capacity alone. Evidence should therefore identify who could speak, what reached decision-makers, which authority was responsible for acting, and what happened when incumbent interests were challenged.

\section{Conclusion}

The literature on AI sovereignty describes dependence on concentrated supplies of compute, chips, cloud services, data, models, energy and expertise. The review examines how institutions convert access to those resources into capabilities that remain usable and open to revision. Ownership and legal control shape that process alongside expertise, decision authority and the practical means to change course.

Collective learning capacity connects authoritative corrigibility with retained option value. Relevant evidence must be able to reach a responsible decision authority, and practical alternatives must remain available when revision is warranted. Building on absorptive capacity, institutional learning and adaptive governance, the five conditions specify where to look for that connection in public technological commitments.

The three cases show different combinations of resources, coordination arrangements and implementation constraints. Their literature documents inputs and formal institutions more extensively than comparable sequences of challenge, authorized response and feasible exit. Targeted comparative inquiry into those sequences would establish how the proposed conditions operate in particular sectors and decisions.

The HCLM application makes the gap between resource indicators and institutional mechanisms explicit. OECD component profiles, IMF preparedness and World Bank trajectories describe different parts of the setting in which learning occurs. They guide investigation of absorption and coordination, while decision records are needed to establish whether evidence reached an authorized response and whether a feasible alternative survived. The review therefore grounds the indicator analysis in an account of what institutions must be able to do, and identifies the additional evidence needed to assess that account.

AI sovereignty therefore includes the institutional ability to understand dependencies, hear affected groups and act on reasons to correct, redirect or refuse a technology. Preserving that ability after deployment is the central problem of collective learning capacity.

\section*{Data and code availability}
\addcontentsline{toc}{section}{Data and code availability}
The ancillary files accompanying this preprint provide the 72-entry source inventory, review protocol, source data and retrieval metadata, and code for the descriptive tables and figures. The quantitative analysis uses published or officially retrieved observations and deterministic descriptive transformations. Primary data remain subject to the terms of the OECD, International Monetary Fund and World Bank sources cited below.

%% file: sections/05_references.tex
\section*{References}
\addcontentsline{toc}{section}{References}
\begingroup\small\setlength{\parskip}{6pt}
\noindent\hangindent=1.4em\hangafter=1 Acemoglu D, Johnson S (2023) Power and progress: our thousand-year struggle over technology and prosperity. PublicAffairs, New York\par
\noindent\hangindent=1.4em\hangafter=1 Ahn E (2026) Who is the human in human-centered AI. AI \& Society 41:4163--4164. \url{https://doi.org/10.1007/s00146-025-02825-6}\par
\noindent\hangindent=1.4em\hangafter=1 Ali SJ, Christin A, Smart A, Katila R (2023) Walking the walk of AI ethics: organizational challenges and the individualization of risk among ethics entrepreneurs. In: Proceedings of the 2023 ACM Conference on Fairness, Accountability, and Transparency. ACM, New York, pp 217--226. \url{https://doi.org/10.1145/3593013.3593990}\par
\noindent\hangindent=1.4em\hangafter=1 Bareis J, Katzenbach C (2022) Talking AI into being: the narratives and imaginaries of national AI strategies and their performative politics. Science, Technology, \& Human Values 47(5):855--881. \url{https://doi.org/10.1177/01622439211030007}\par
\noindent\hangindent=1.4em\hangafter=1 Belli L (2026) Operationalizing AI sovereignty through agency, interoperability, and openness. Carnegie Endowment for International Peace, 3 September. \url{https://carnegieendowment.org/research/2026/09/operationalizing-ai-sovereignty-through-agency-interoperability-and-openness}. Accessed 8 September 2026\par
\noindent\hangindent=1.4em\hangafter=1 Boin A, van Eeten MJG (2013) The resilient organization: a critical appraisal. Public Management Review 15(3):429--445. \url{https://doi.org/10.1080/14719037.2013.769856}\par
\noindent\hangindent=1.4em\hangafter=1 Bradford A (2023) Digital empires: the global battle to regulate technology. Oxford University Press, New York\par
\noindent\hangindent=1.4em\hangafter=1 Cachipuendo-Contero IJ (2026) The epistemic readiness gap: rethinking AI readiness indices through ILIA 2025. AI \& Society 41:7433--7446. \url{https://doi.org/10.1007/s00146-026-03087-6}\par
\noindent\hangindent=1.4em\hangafter=1 Calzati S, van Loenen B (2025) Beyond federated data: a data commoning proposition for the EU's citizen-centric digital strategy. AI \& Society 40:945--957. \url{https://doi.org/10.1007/s00146-023-01743-9}\par
\noindent\hangindent=1.4em\hangafter=1 Castellacci F, Natera JM (2013) The dynamics of national innovation systems: a panel cointegration analysis of the coevolution between innovative capability and absorptive capacity. Research Policy 42(3):579--594. \url{https://doi.org/10.1016/j.respol.2012.10.006}\par
\noindent\hangindent=1.4em\hangafter=1 Cath C, Wachter S, Mittelstadt B, Taddeo M, Floridi L (2018) Artificial intelligence and the 'good society': the US, EU, and UK approach. Science and Engineering Ethics 24(2):505--528. \url{https://doi.org/10.1007/s11948-017-9901-7}\par
\noindent\hangindent=1.4em\hangafter=1 Cazzaniga M, Jaumotte F, Li L, Melina G, Panton AJ, Pizzinelli C, Rockall EJ, Tavares MM (2024) Gen-AI: artificial intelligence and the future of work. IMF Staff Discussion Note 2024/001. \url{https://doi.org/10.5089/9798400262548.006}\par
\noindent\hangindent=1.4em\hangafter=1 Coeckelbergh M (2025) The case for global governance of AI: arguments, counter-arguments, and challenges ahead. AI \& Society 40:1803--1806. \url{https://doi.org/10.1007/s00146-024-01949-5}\par
\noindent\hangindent=1.4em\hangafter=1 Cohen WM, Levinthal DA (1990) Absorptive capacity: a new perspective on learning and innovation. Administrative Science Quarterly 35(1):128--152. \url{https://doi.org/10.2307/2393553}\par
\noindent\hangindent=1.4em\hangafter=1 Cour des comptes (2025) The national strategy for artificial intelligence. 19 November. \url{https://www.ccomptes.fr/sites/default/files/2025-11/20251119-The-National-Strategy-for-Artificial-Intelligence.pdf}. Accessed 28 September 2026\par
\noindent\hangindent=1.4em\hangafter=1 Couture S, Toupin S (2019) What does the notion of 'sovereignty' mean when referring to the digital? New Media \& Society 21(10):2305--2322. \url{https://doi.org/10.1177/1461444819865984}\par
\noindent\hangindent=1.4em\hangafter=1 Crawford K (2021) Atlas of AI: power, politics, and the planetary costs of artificial intelligence. Yale University Press, New Haven\par
\noindent\hangindent=1.4em\hangafter=1 Edler J, Blind K, Kroll H, Schubert T (2023) Technology sovereignty as an emerging frame for innovation policy. Defining rationales, ends and means. Research Policy 52(6):104765. \url{https://doi.org/10.1016/j.respol.2023.104765}\par
\noindent\hangindent=1.4em\hangafter=1 Edler J, Fagerberg J (2017) Innovation policy: what, why, and how. Oxford Review of Economic Policy 33(1):2--23. \url{https://doi.org/10.1093/oxrep/grx001}\par
\noindent\hangindent=1.4em\hangafter=1 European Parliament, Council of the European Union (2024) Regulation (EU) 2024/1689 of the European Parliament and of the Council of 13 June 2024 laying down harmonised rules on artificial intelligence and amending Regulations (EC) No 300/2008, (EU) No 167/2013, (EU) No 168/2013, (EU) 2018/858, (EU) 2018/1139 and (EU) 2019/2144 and Directives 2014/90/EU, (EU) 2016/797 and (EU) 2020/1828 (Artificial Intelligence Act) (Text with EEA relevance). Official Journal of the European Union L 2024/1689, 12 July 2024. \url{http://data.europa.eu/eli/reg/2024/1689/oj}\par
\noindent\hangindent=1.4em\hangafter=1 Eurostat (2025) 20\% of EU enterprises use AI technologies. \url{https://ec.europa.eu/eurostat/web/products-eurostat-news/w/ddn-20251211-2}. Accessed 8 September 2026\par
\noindent\hangindent=1.4em\hangafter=1 Farrell H, Newman AL (2019) Weaponized interdependence: how global economic networks shape state coercion. International Security 44(1):42--79. \url{https://doi.org/10.1162/isec_a_00351}\par
\noindent\hangindent=1.4em\hangafter=1 Floridi L (2020) The fight for digital sovereignty: what it is, and why it matters, especially for the EU. Philosophy \& Technology 33:369--378. \url{https://doi.org/10.1007/s13347-020-00423-6}\par
\noindent\hangindent=1.4em\hangafter=1 Floridi L, Cowls J, Beltrametti M, et al (2018) AI4People: an ethical framework for a good AI society: opportunities, risks, principles, and recommendations. Minds and Machines 28:689--707. \url{https://doi.org/10.1007/s11023-018-9482-5}\par
\noindent\hangindent=1.4em\hangafter=1 Folke C, Hahn T, Olsson P, Norberg J (2005) Adaptive governance of social-ecological systems. Annual Review of Environment and Resources 30:441--473. \url{https://doi.org/10.1146/annurev.energy.30.050504.144511}\par
\noindent\hangindent=1.4em\hangafter=1 Fratini S, Hine E, Novelli C, Roberts H, Floridi L (2024) Digital sovereignty: a descriptive analysis and a critical evaluation of existing models. Digital Society 3:59. \url{https://doi.org/10.1007/s44206-024-00146-7}\par
\noindent\hangindent=1.4em\hangafter=1 Haag A (2025) The state of AI competition in advanced economies. FEDS Notes, Board of Governors of the Federal Reserve System, 6 October. \url{https://doi.org/10.17016/2380-7172.3930}\par
\noindent\hangindent=1.4em\hangafter=1 Hine E, Floridi L (2024) Artificial intelligence with American values and Chinese characteristics: a comparative analysis of American and Chinese governmental AI policies. AI \& Society 39(1):257--278. \url{https://doi.org/10.1007/s00146-022-01499-8}\par
\noindent\hangindent=1.4em\hangafter=1 Hummel P, Braun M, Tretter M, Dabrock P (2021) Data sovereignty: a review. Big Data \& Society 8(1):1--17. \url{https://doi.org/10.1177/2053951720982012}\par
\noindent\hangindent=1.4em\hangafter=1 Hälterlein J (2025) Imagining and governing artificial intelligence: the ordoliberal way: an analysis of the national strategy 'AI made in Germany'. AI \& Society 40:1749--1760. \url{https://doi.org/10.1007/s00146-024-01940-0}\par
\noindent\hangindent=1.4em\hangafter=1 International Monetary Fund (2024) AI Preparedness Index. IMF DataMapper. \url{https://www.imf.org/external/datamapper/datasets/AIPI}. Accessed 20 October 2026\par
\noindent\hangindent=1.4em\hangafter=1 Jasanoff S (ed) (2004) States of knowledge: the co-production of science and social order. Routledge, London\par
\noindent\hangindent=1.4em\hangafter=1 Jasanoff S, Kim S-H (eds) (2015) Dreamscapes of modernity: sociotechnical imaginaries and the fabrication of power. University of Chicago Press, Chicago\par
\noindent\hangindent=1.4em\hangafter=1 Jobin A, Ienca M, Vayena E (2019) The global landscape of AI ethics guidelines. Nature Machine Intelligence 1:389--399. \url{https://doi.org/10.1038/s42256-019-0088-2}\par
\noindent\hangindent=1.4em\hangafter=1 Kaack LH, Donti PL, Strubell E, Kamiya G, Creutzig F, Rolnick D (2022) Aligning artificial intelligence with climate change mitigation. Nature Climate Change 12:518--527. \url{https://doi.org/10.1038/s41558-022-01377-7}\par
\noindent\hangindent=1.4em\hangafter=1 Keerthiraj, Misra A (2026) Who owns the future? AI, digital sovereignty, and the politics of knowledge. AI \& Society 41:2143--2155. \url{https://doi.org/10.1007/s00146-025-02641-y}\par
\noindent\hangindent=1.4em\hangafter=1 Liebig L, Güttel L, Jobin A, Katzenbach C (2024) Subnational AI policy: shaping AI in a multi-level governance system. AI \& Society 39:1477--1490. \url{https://doi.org/10.1007/s00146-022-01561-5}\par
\noindent\hangindent=1.4em\hangafter=1 Lundvall B-Å (ed) (1992) National systems of innovation: towards a theory of innovation and interactive learning. Pinter, London\par
\noindent\hangindent=1.4em\hangafter=1 Mazzucato M (2013) The entrepreneurial state: debunking public vs. private sector myths. Anthem Press, London\par
\noindent\hangindent=1.4em\hangafter=1 Mueller ML (2020) Against sovereignty in cyberspace. International Studies Review 22(4):779--801. \url{https://doi.org/10.1093/isr/viz044}\par
\noindent\hangindent=1.4em\hangafter=1 Mügge D (2024) EU AI sovereignty: for whom, to what end, and to whose benefit? Journal of European Public Policy 31(8):2200--2225. \url{https://doi.org/10.1080/13501763.2024.2318475}\par
\noindent\hangindent=1.4em\hangafter=1 National Conference of State Legislatures (2025) Artificial intelligence 2025 legislation. Updated 10 July. \url{https://www.ncsl.org/technology-and-communication/artificial-intelligence-2025-legislation}. Accessed 28 September 2026\par
\noindent\hangindent=1.4em\hangafter=1 Novelli C, Taddeo M, Floridi L (2024) Accountability in artificial intelligence: what it is and how it works. AI \& Society 39:1871--1882. \url{https://doi.org/10.1007/s00146-023-01635-y}\par
\noindent\hangindent=1.4em\hangafter=1 OECD (2024) Recommendation of the Council on Artificial Intelligence (OECD/LEGAL/0449, amended 2024). \url{https://legalinstruments.oecd.org/en/instruments/OECD-LEGAL-0449}. Accessed 8 September 2026\par
\noindent\hangindent=1.4em\hangafter=1 OECD (2026a) AI use by individuals surges across the OECD as adoption by firms continues to expand. 28 January. \url{https://www.oecd.org/en/about/news/announcements/2026/01/ai-use-by-individuals-surges-across-the-oecd-as-adoption-by-firms-continues-to-expand.html}. Accessed 8 September 2026\par
\noindent\hangindent=1.4em\hangafter=1 OECD (2026b) OECD Economic Surveys: France 2026. OECD Publishing, Paris. \url{https://doi.org/10.1787/e88a1716-en}\par
\noindent\hangindent=1.4em\hangafter=1 OECD (2026c) The OECD.AI Index. OECD Publishing, Paris. \url{https://doi.org/10.1787/32c01014-en}\par
\noindent\hangindent=1.4em\hangafter=1 Opara-Martins J, Sahandi R, Tian F (2016) Critical analysis of vendor lock-in and its impact on cloud computing migration: a business perspective. Journal of Cloud Computing 5:4. \url{https://doi.org/10.1186/s13677-016-0054-z}\par
\noindent\hangindent=1.4em\hangafter=1 Ostrom E (1990) Governing the commons: the evolution of institutions for collective action. Cambridge University Press, Cambridge\par
\noindent\hangindent=1.4em\hangafter=1 Plantin J-C, Lagoze C, Edwards PN, Sandvig C (2018) Infrastructure studies meet platform studies in the age of Google and Facebook. New Media \& Society 20(1):293--310. \url{https://doi.org/10.1177/1461444816661553}\par
\noindent\hangindent=1.4em\hangafter=1 Pohle J, Nanni R, Santaniello M (2024) Unthinking digital sovereignty: a critical reflection on origins, objectives, and practices. Policy \& Internet 16(4):666--671. \url{https://doi.org/10.1002/poi3.437}\par
\noindent\hangindent=1.4em\hangafter=1 Pohle J, Thiel T (2020) Digital sovereignty. Internet Policy Review 9(4). \url{https://doi.org/10.14763/2020.4.1532}\par
\noindent\hangindent=1.4em\hangafter=1 Privette AP, Barros A, Cai X (2026) Data centers water footprint: the need for more transparency. AGU Advances 7:e2025AV002140. \url{https://doi.org/10.1029/2025AV002140}\par
\noindent\hangindent=1.4em\hangafter=1 Radu R (2021) Steering the governance of artificial intelligence: national strategies in perspective. Policy and Society 40(2):178--193. \url{https://doi.org/10.1080/14494035.2021.1929728}\par
\noindent\hangindent=1.4em\hangafter=1 Roberts H (2024) Digital sovereignty and artificial intelligence: a normative approach. Ethics and Information Technology 26:70. \url{https://doi.org/10.1007/s10676-024-09810-5}\par
\noindent\hangindent=1.4em\hangafter=1 Roberts H, Cowls J, Casolari F, Morley J, Taddeo M, Floridi L (2021a) Safeguarding European values with digital sovereignty: an analysis of statements and policies. Internet Policy Review 10(3). \url{https://doi.org/10.14763/2021.3.1575}\par
\noindent\hangindent=1.4em\hangafter=1 Roberts H, Cowls J, Morley J, Taddeo M, Wang V, Floridi L (2021b) The Chinese approach to artificial intelligence: an analysis of policy, ethics, and regulation. AI \& Society 36(1):59--77. \url{https://doi.org/10.1007/s00146-020-00992-2}\par
\noindent\hangindent=1.4em\hangafter=1 Selbst AD, boyd d, Friedler SA, Venkatasubramanian S, Vertesi J (2019) Fairness and abstraction in sociotechnical systems. In: Proceedings of the 2019 Conference on Fairness, Accountability, and Transparency. ACM, New York, pp 59--68. \url{https://doi.org/10.1145/3287560.3287598}\par
\noindent\hangindent=1.4em\hangafter=1 Sen A (1999) Development as freedom. Oxford University Press, Oxford\par
\noindent\hangindent=1.4em\hangafter=1 Sheikh H (2022) European digital sovereignty: a layered approach. Digital Society 1:25. \url{https://doi.org/10.1007/s44206-022-00025-z}\par
\noindent\hangindent=1.4em\hangafter=1 Shneiderman B (2022) Human-centered AI. Oxford University Press, New York\par
\noindent\hangindent=1.4em\hangafter=1 Snyder H (2019) Literature review as a research methodology: an overview and guidelines. Journal of Business Research 104:333--339. \url{https://doi.org/10.1016/j.jbusres.2019.07.039}\par
\noindent\hangindent=1.4em\hangafter=1 Standing Committee of the National People's Congress (2021) Personal Information Protection Law of the People's Republic of China. \url{https://en.spp.gov.cn/2021-12/29/c_948419_2.htm}. Accessed 28 September 2026\par
\noindent\hangindent=1.4em\hangafter=1 Stanford Institute for Human-Centered Artificial Intelligence (2026) The 2026 AI Index Report. Stanford University. \url{https://hai.stanford.edu/ai-index/2026-ai-index-report}. Accessed 8 September 2026\par
\noindent\hangindent=1.4em\hangafter=1 State Council of the People's Republic of China (2023) China inaugurates National Data Administration. 25 October. \url{https://english.www.gov.cn/news/202310/25/content_WS6538ae0dc6d0868f4e8e0a30.html}. Accessed 28 September 2026\par
\noindent\hangindent=1.4em\hangafter=1 Taylor RR, Murphy JW, Hoston WT, Senkaiahliyan S (2025) Democratizing AI in public administration: improving equity through maximum feasible participation. AI \& Society 40:3653--3662. \url{https://doi.org/10.1007/s00146-024-02120-w}\par
\noindent\hangindent=1.4em\hangafter=1 Torraco RJ (2005) Writing integrative literature reviews: guidelines and examples. Human Resource Development Review 4(3):356--367. \url{https://doi.org/10.1177/1534484305278283}\par
\noindent\hangindent=1.4em\hangafter=1 Tran KP (2026) Human-Centered Learning Mechanics -- Control, Calibration and Governed Action for Industrial AI Decision Systems. SpringerBriefs in Intelligent Systems. Springer, forthcoming.\par
\noindent\hangindent=1.4em\hangafter=1 UNESCO (2022) Recommendation on the ethics of artificial intelligence. \url{https://unesdoc.unesco.org/ark:/48223/pf0000381137}. Accessed 8 September 2026\par
\noindent\hangindent=1.4em\hangafter=1 Whittemore R, Knafl K (2005) The integrative review: updated methodology. Journal of Advanced Nursing 52(5):546--553. \url{https://doi.org/10.1111/j.1365-2648.2005.03621.x}\par
\noindent\hangindent=1.4em\hangafter=1 World Bank (2026) World Development Indicators [data set]. \url{https://databank.worldbank.org/source/world-development-indicators}. Accessed 20 October 2026\par
\noindent\hangindent=1.4em\hangafter=1 Zuiderwijk A, Chen Y-C, Salem F (2021) Implications of the use of artificial intelligence in public governance: a systematic literature review and a research agenda. Government Information Quarterly 38(3):101577. \url{https://doi.org/10.1016/j.giq.2021.101577}\par
\endgroup